%% file: arxiv.tex
\documentclass{article}

\usepackage{microtype}
\usepackage{graphicx}
\usepackage{subcaption}
\usepackage{booktabs} %
\usepackage{multirow}
\usepackage{url}

\usepackage{hyperref}

\usepackage[preprint]{icml2026}

\usepackage{amsmath}
\usepackage{amssymb}
\usepackage{mathtools}
\usepackage{amsthm}
\usepackage[most]{tcolorbox}

\usepackage[most]{tcolorbox}
\newtcolorbox{promptbox}[1]{
    colback=gray!5!white,     %
    colframe=gray!75!black,   %
    fonttitle=\bfseries\sffamily,
    title=#1,                 %
    enhanced,
    attach boxed title to top left={yshift=-2mm, xshift=2mm},
    boxed title style={colback=gray!75!black},
    sharp corners=south,      %
    breakable,                %
    fontupper=\small\ttfamily %
}

\usepackage{tabularx}
\usepackage{booktabs}

\newtcolorbox{resultsbox}[1]{
    colback=white,
    colframe=gray!20!black,
    fonttitle=\bfseries\sffamily,
    title=#1,
    enhanced,
    attach boxed title to top left={yshift=-2mm, xshift=2mm},
    boxed title style={colback=gray!20!black},
    sharp corners,
    drop shadow, %
    breakable
}

\theoremstyle{plain}
\newtheorem{theorem}{Theorem}[section]
\newtheorem{proposition}[theorem]{Proposition}

\theoremstyle{definition}

\newtheorem{assumption}[theorem]{Assumption}
\theoremstyle{remark}

\newcommand{\myref}[2]{%
    \hyperref[#2]{#1~\ref*{#2}}%
}

\usepackage[textsize=tiny]{todonotes}

\icmltitlerunning{Causal Effect Estimation with Latent Textual Treatments}

\usepackage{etoolbox}
\AtBeginEnvironment{align}{\setlength{\abovedisplayskip}{5pt} \setlength{\belowdisplayskip}{5pt}}
\AtBeginEnvironment{equation}{\setlength{\abovedisplayskip}{5pt} \setlength{\belowdisplayskip}{5pt}}
 
\begin{document}

\twocolumn[
  \icmltitle{Causal Effect Estimation with Latent Textual Treatments}

  \icmlsetsymbol{equal}{*}

  \begin{icmlauthorlist}
    \icmlauthor{Omri Feldman}{huji}
    \icmlauthor{Amar Venugopal}{stanfordecon}
    \icmlauthor{Jann Spiess}{stanfordbusiness}
    \icmlauthor{Amir Feder}{huji}
  \end{icmlauthorlist}

  \icmlaffiliation{huji}{School of Computer Science and Engineering, The Hebrew University of Jerusalem}
  \icmlaffiliation{stanfordecon}{Department of Economics, Stanford University}
  \icmlaffiliation{stanfordbusiness}{Graduate School of Business, Stanford University}

  \icmlcorrespondingauthor{Omri Feldman}{omri.feldman@mail.huji.ac.il}

  \icmlkeywords{Machine Learning, ICML}

  \vskip 0.3in
]

\printAffiliationsAndNotice{}

\begin{abstract}

Understanding the causal effects of text on downstream outcomes is a central task in many applications.
Estimating such effects requires researchers to run controlled experiments that systematically vary textual features. While large language models (LLMs) hold promise for generating text, producing and evaluating \emph{controlled} variation requires more careful attention.
In this paper, we present an end-to-end pipeline for the generation and causal estimation of latent textual interventions. Our work first performs hypothesis generation and steering via sparse autoencoders (SAEs), followed by robust causal estimation. Our pipeline addresses both computational and statistical challenges in text-as-treatment experiments.
We demonstrate that naive estimation of causal effects suffers from significant bias as text inherently conflates treatment and covariate information. We describe the estimation bias induced in this setting and propose a solution based on covariate residualization.
Our empirical results show that our pipeline effectively induces variation in target features and mitigates estimation error, providing a robust foundation for causal effect estimation in text-as-treatment settings.
\end{abstract}

\section{Introduction}

Whether it is the aspects of political speech that are effective in creating voter support or how social media posts can be most engaging, understanding the causal effect of textual features on downstream outcomes is of great interest across domains \cite{gerber2011large, fenerty2012effect,  egan2013skilled, fong2016discovery, green2019get, luong2021text, ash2023text, ellickson2023estimating}.
However, researchers face two particular challenges when designing experimental pipelines for such questions. First, they need to be able to modify text in systematic ways, while ensuring that documents remain readable. Second, valid causal estimation requires a separation  between those intended modifications and other variations in the text \cite{feder2022causal}.
While large language models (LLMs) have revolutionized our ability to generate texts at scale, addressing both of the above challenges requires several key methodological enhancements.

\begin{figure}
    \centering
    \includegraphics[width=0.95\linewidth]{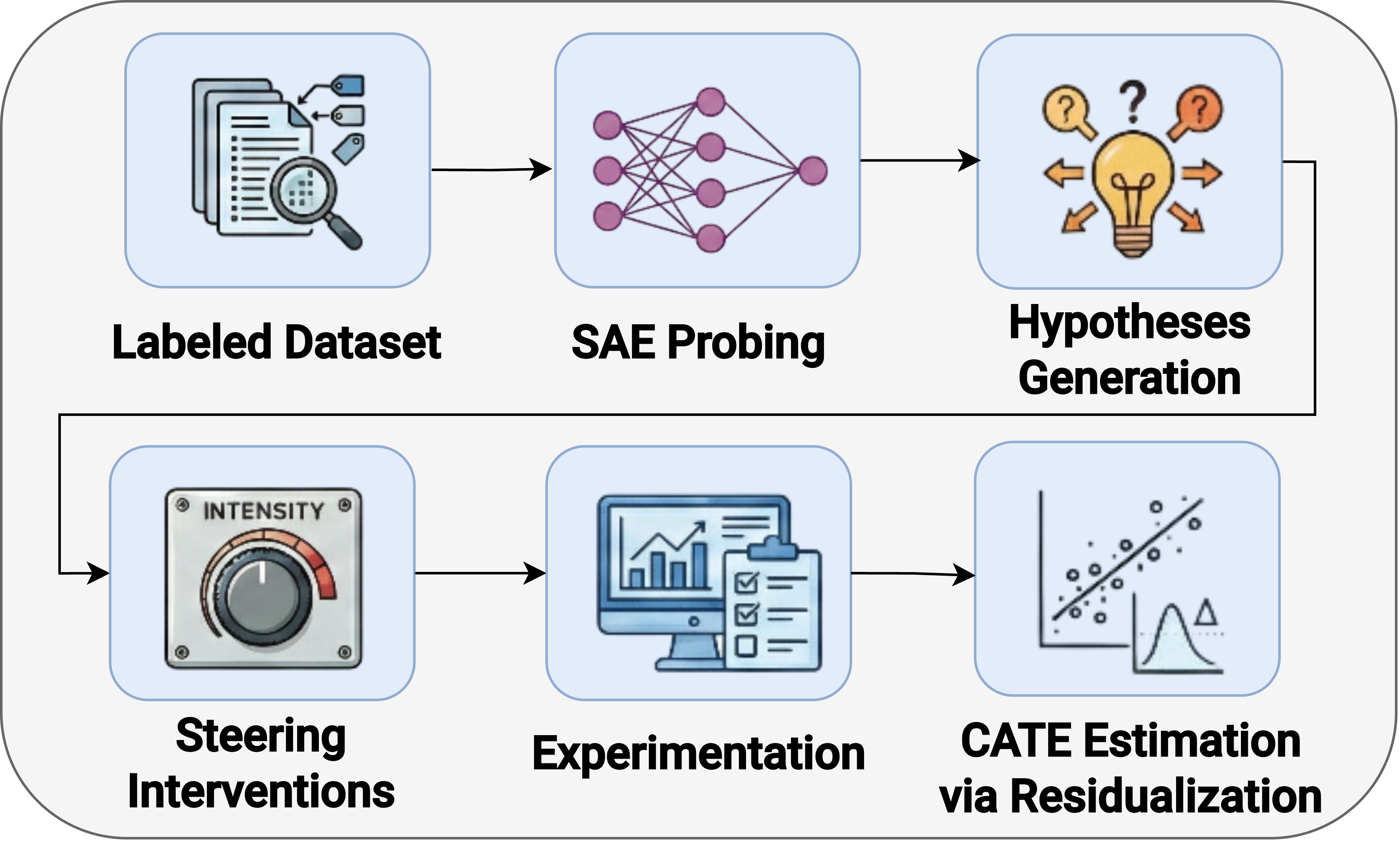}
    \caption{An overview of our methodology. We start with a dataset of text documents, with labels classifying the semantic concept we want to intervene on. We run probes on SAE representations and identify the most persistent and semantically relevant SAE features that serve as our hypotheses. We then steer texts to generate quasi-counterfactuals that are used in downstream experiments. Finally, we present a novel residualization approach for CATE estimation in such text-as-treatment setting, alongside theoretical guarantees and empirical evidence.}
    \label{fig:birdseye}
    \vspace{-15pt}
\end{figure}

In this paper, we propose a novel framework for applying controlled interventions on textual features and estimating their causal effect on downstream outcomes.
We build on the high-dimensional representation of language in LLMs to isolate semantic features and systematically intervene on those representations, producing text that varies primarily in targeted features. To facilitate a methodical choice between various candidate features, we also introduce a scoring metric that quantifies the effectiveness of our interventions and the potential undesired confounding artifacts it may introduce. We then estimate causal effects on downstream outcomes using methods from causal machine learning.

To understand the appeal of this approach, consider the following example. \citet{martinvenugopal} recently assembled a dataset of transcripts of local government meetings across various towns in California. A key feature of these meetings is a public comment period, in which members of the public may address the municipal government directly. These comments can vary widely in their ``civility", for which a score is provided for every observation in the dataset. A social-science researcher may be interested in understanding which aspects of speech result in a comment being regarded as more or less civil. Since civility is a complex concept, there may be many atomic features in the text that can impact the perceived civility of the comment. Our framework provides a researcher with a tool to generate hypotheses about which atomic features impact perceived civility and, subsequently, generate quasi-counterfactual comments that vary in these features. 

The researcher's goal may then become to validate these hypotheses through a randomized experiment, in which these quasi-counterfactual texts are shown to respondents who evaluate their civility. Similarly, the researcher may instead be interested in seeing if changes in concepts in a text that are correlated with civility impact some alternative downstream outcome of interest, such as respondent agreement. Assessing causality in such a setting presents unique challenges. While randomization of steered texts to respondents will ensure potential outcomes are uncorrelated with treatment assignment, steering itself is rarely ``surgical", and will likely perturb correlated nuisance features. As a result, steered texts are not true counterfactuals, and we control for such unwanted variations by adjusting for dense contextual embeddings using robust causal machine learning.

In implementing our approach to causal estimation, we must address a fundamental challenge to causal identification: since the treatments (along with possible confounders) are expressed in the text, embeddings produced by a sufficiently expressive model will perfectly predict treatment itself. Using such embeddings as controls will (in the parlance of causal inference) lead to a positivity violation, with no overlap between treated and control documents and degenerate propensity scores. This leads to a central trade-off: controlling for more sophisticated text embeddings destroys variation in treatment, while weaker embedding models fail to provide sufficient control for nuisance features, leading to omitted-variable bias.

To resolve this tension, we propose a novel residualization procedure which removes treatment information from the embeddings before proceeding with treatment effect estimation. We provide theoretical results that establish causal identification and bound the resulting estimation error. We then validate this approach in practice via semi-synthetic simulations and show consistent, marked improvements over naive estimation procedures that control for full embeddings directly, in line with our theoretical guarantees (see example in \myref{Figure}{fig:cate_estimation}). Our approach constitutes a robust foundation for causal effect estimation with latent treatments, and can be extended naturally to other modalities.

\begin{figure}[ht]
    \centering
    \includegraphics[width=0.65\linewidth]{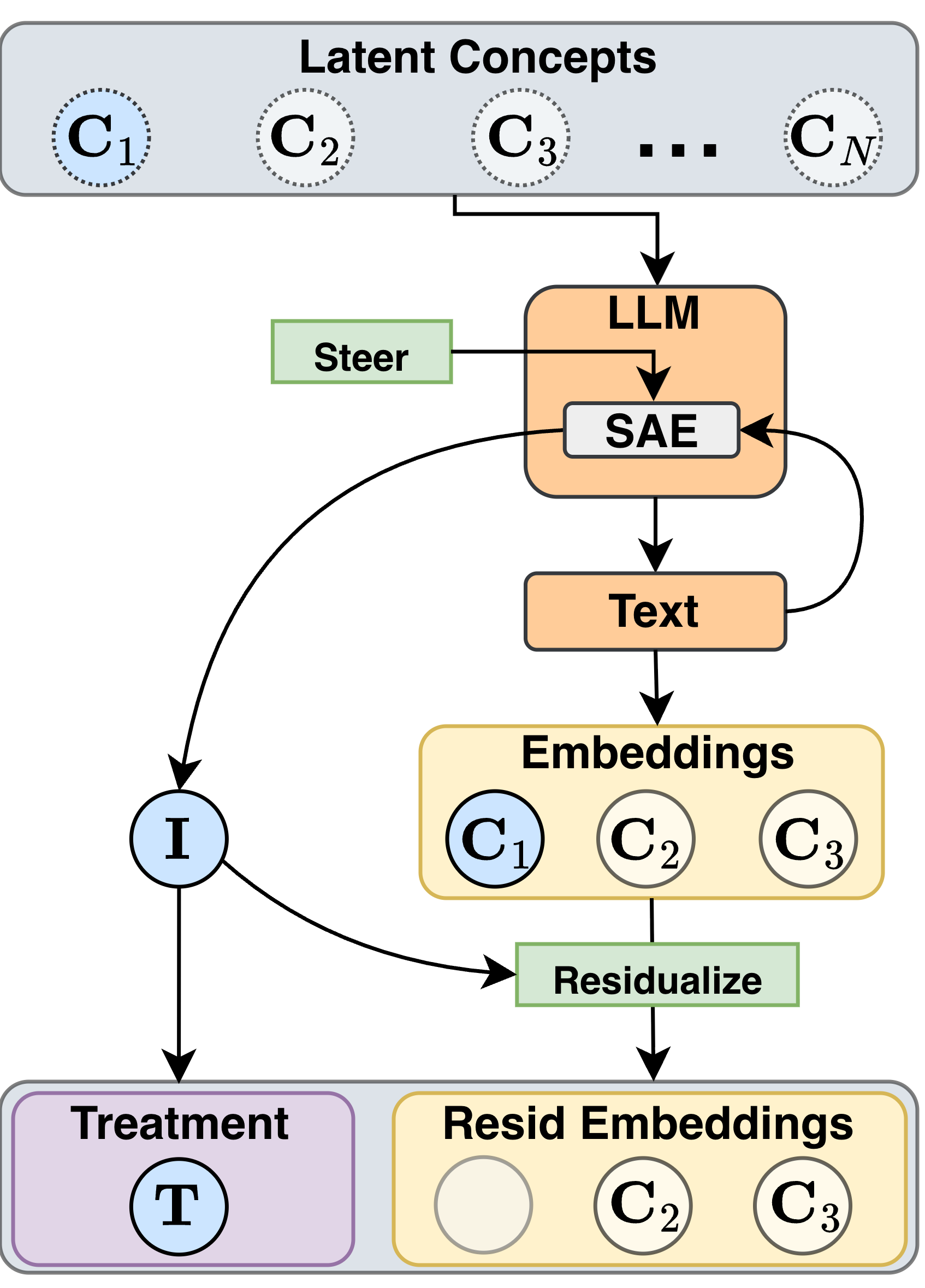}
    \caption{Our residualization pipeline. We assume that natural language is generated from latent semantic concepts, which serve as inputs to the LLM. By applying SAE steering interventions, we modify the model’s generation process and produce quasi-counterfactual texts. We then measure their ex-post intensity $I$, and use embedding model for representing the text use as controls. Finally, using our measured intensity, we residualize the embeddings to remove the target concept $\mathbf{C}_1$, which ensures that the treatment information is separated from the rest of text.}
    \label{fig:steerresidflow}
    \vspace{-15pt}
\end{figure}

\begin{figure}[h]
    \centering
    \vspace{-5pt}
    \includegraphics[width=0.8\linewidth]{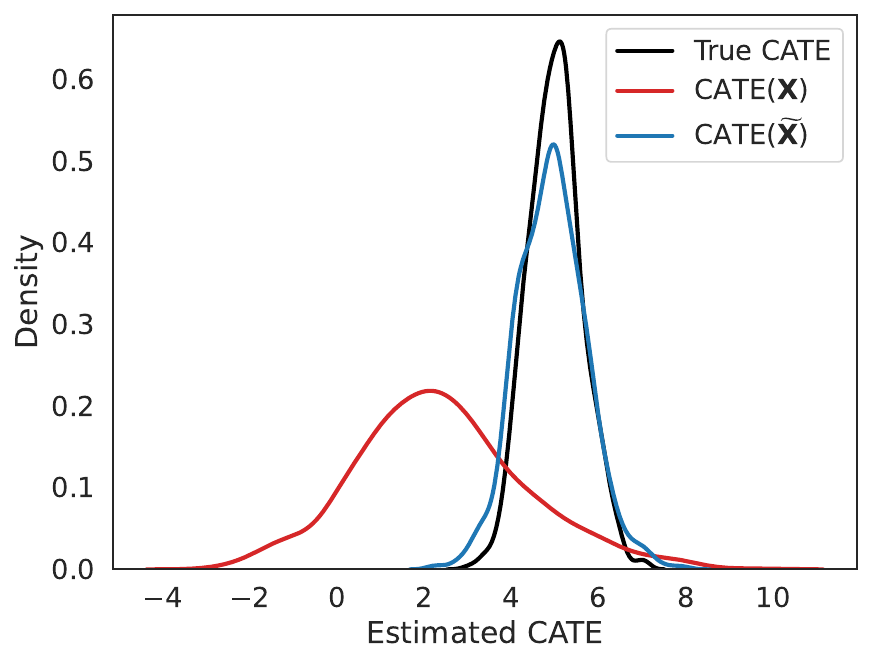}
    \includegraphics[width=0.8\linewidth]{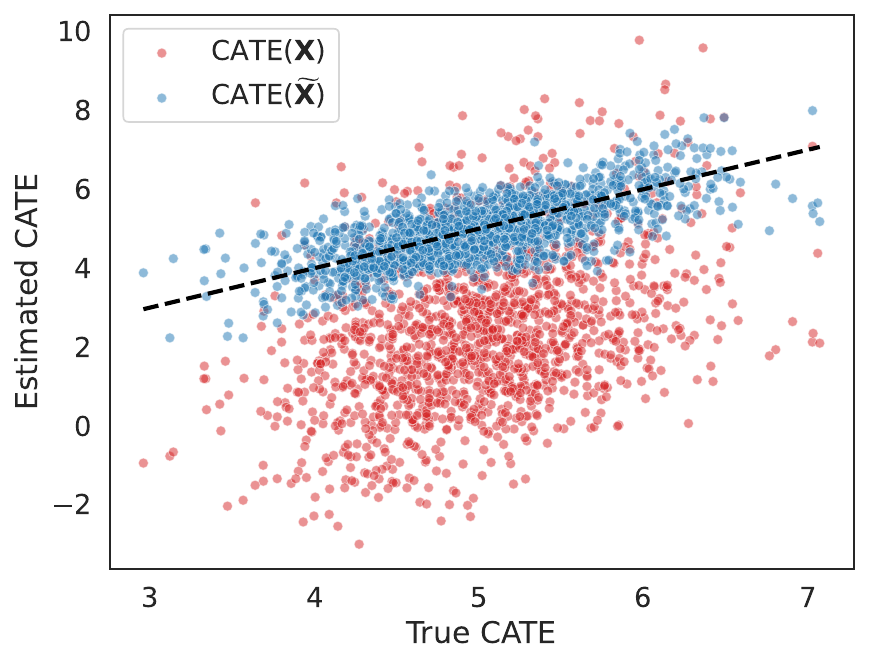}
    \caption{CATE simulation from Dataset C, using Llama-3.1-8B-Instruct, layer 23 feature 53435, and all-MiniLM-L6-v2 embeddings. Top: distribution of true (black) and estimated CATE based on raw (red) and residualized (blue) covariates. Bottom: scatter of true vs. estimated CATE values based on raw (red) and residualized (blue) covariates, with the 45-degree line in black. Residualization performed via dimension-by-dimension strategy.}
    \label{fig:cate_estimation}
    \vspace{-15pt}
\end{figure}

\section{Preliminaries}

In this section, we formalize the notation for Sparse Autoencoders and the causal inference framework used throughout.

\subsection{Sparse Autoencoders (SAEs)}
Sparse Autoencoders are designed to decompose the dense, polysemantic activations of an LLM into a high-dimensional, sparse latent space that corresponds more closely to interpretable features \citep{bricken2023monosemanticity}.

Let $a \in \mathbb{R}^d$ be the vector of activations from a specific layer of a language model. An SAE consists of an encoder $W_{enc} \in \mathbb{R}^{h \times d}$ and a decoder $W_{dec} \in \mathbb{R}^{d \times h}$, where the latent dimension $h$ is typically much larger than the hidden dimension $d$. The encoder maps the activations into a sparse feature vector $z \in \mathbb{R}^h$ and the decoder maps $z$ back to $\mathbb{R}^d$,
\begin{align}
    &z = \sigma(W_{enc}a + b_{enc}), %
    &
    &\hat{a} = W_{dec}z + b_{dec} \label{eq:saedec},
\end{align}
where $\sigma(\cdot)$ is a nonlinear activation function (e.g. ReLU, JumpReLU, TopK) and $b_{enc}$ and $b_{dec}$ are bias terms.
During training, a sparsity penalty is applied to $z$ to ensure that only a small number of features are active for any given input. We denote the indices of the SAE latent space as $\phi \in \Phi$, latent feature activation as $z_\phi$, and the $\phi$-th column of the decoder matrix as $W_{dec(\phi)}$.

\subsection{Conditional Average Treatment Effects (CATEs)}
\label{sec:TE-Intro}

Our ultimate goal is to estimate the causal effect of controlled variation in text features on downstream outcomes of interest.
We adopt the potential-outcomes framework from causal inference.
We write $T \in \{0, 1\}$ for a binary treatment variable (e.g., the presence of a specific linguistic feature) and $Y$ for an observed downstream outcome. For a given unit, we consider potential outcomes $Y(1)$ and $Y(0)$, representing the outcome under treatment and control, respectively.
The realized outcome is $Y = Y(T)$.
The individual treatment effect is $Y(1) - Y(0)$, which is unobservable. We aim to estimate the Conditional Average Treatment Effect (CATE)
\begin{equation}
\label{eqn:CATE}
    \tau(\tilde{x}) = \mathbb{E}[Y(1) - Y(0) \mid \tilde{X} {=} \tilde{x}],
\end{equation}
which is the expected effect of the treatment given a set of covariates $\tilde{X}{=}\tilde{x}$.
In the context of \textit{text-as-treatment}, $\tilde{X}$ typically consists of high-dimensional contextual embeddings that capture the nuisance properties of the text (e.g., length, base semantic content, style). Standard identification of $\tau(\tilde{x})$ requires two fundamental assumptions:

\vspace{-10pt}
\begin{enumerate}
\itemsep2pt
    \item \textbf{Unconfoundedness (Ignorability):} The potential outcomes are independent of the treatment assignment given the covariates, $(Y(0), Y(1)) \perp T \mid \tilde{X}$.
    \item \textbf{Positivity (Overlap):} For all $\tilde{x}$ in the support of $\tilde{X}$, the probability of receiving the treatment is bounded away from 0 and 1: $0 < P(T=1 \mid \tilde{X} {=} \tilde{x}) < 1$.
\end{enumerate}
\vspace{-10pt}

Throughout, we will maintain the additional assumption that units $i$ with observations $(Y_i, T_i, \tilde{X}_i)$ are drawn i.i.d from a population distribution. In particular, the SUTVA assumption of no interference is implicit: potential outcomes for a given unit are independent of the treatment status of other units, $(Y_i(0), Y_i(1)) \perp T_{j\neq i}$.

\section{Related Work}

\subsection{SAEs and Steering}

SAEs have been shown to be useful in decomposing the polysemantic representations of LLMs into monosemantic, human-interpretable features \cite{Bills2023Autointerp, bricken2023monosemanticity, cunningham2023sparseautoencodershighlyinterpretable, templeton2024scaling}. Recent research has leveraged the granularity of these features for steering and hypothesis generation \cite{movva2025sparseautoencodershypothesisgeneration, peng2025usesparseautoencodersdiscover, zheng2025model}. However, alongside their usefulness for interpretability, benchmark results \cite{wu2025axbenchsteeringllmssimple, huang2024ravel, kantamneni2025sparseautoencodersusefulcase} suggest that SAEs and SAE-based steering perform poorly compared to other methods. Our work extends this literature by proposing a new metric based on the model's response to a steering intervention, establishing a more rigorous link between a feature's interpretability and its utility in downstream causal experiments.

\subsection{Causal Inference with Text}

Estimating causal quantities with textual data presents multiple unique challenges \cite{feder2022causal, egami2022make}. Recent work addressed these challenges when the textual information is used as controls \cite{veitch2020adapting, mozer2020matching}, treatment \cite{fong2016discovery, pryzant2021causal, gui2022causal, ash2023text, angelopoulos2024value}, and outcome \cite{modarressi2025causal}. Particularly relevant to our work is \citet{imai2025causalrepresentationlearninggenerative}, which studies a similar setting under slightly different assumptions. We contribute to this literature by presenting a novel steering pipeline that enables controlled interventions and addresses inherent positivity violations in the text-as-treatment setting.

\section{Datasets}
\label{sec:Datasets}

Throughout this paper, we present empirical results from three labeled text datasets.

\textbf{Dataset A: Local Government Speech.} A dataset of city council meetings in California produced by \citet{martinvenugopal}. The dataset is comprised of statements made by members of the public in city council meetings across California. The label corresponds to a binarized measure of ``civility": a label of 1 indicates a highly civil statement, while a label of 0 corresponds to a highly uncivil statement.

\textbf{Dataset B: Political Ads.} A dataset of US political advertisements, as presented in \citet{sobhani2024multienvironmenttopicmodels}.
A label of 1 indicates a Republican politician and a label of 0 indicates a Democratic politician.

\textbf{Dataset C: Us Vs. Them.} A dataset of reddit comments posted
in response to news articles across the political
spectrum, curated by \citet{huguet-cabot-etal-2021-us}. The comments were scored for populist attitudes towards out-groups, which we binarize such that a label of 1 indicates a discriminatory comment towards the out-group and a label of 0 indicates a supportive comment towards the out-group.

\section{Hypothesis Generation}

The hypothesis generation phase identifies a set of candidate SAE features $\Phi_+ \subset \Phi$ that differentiate between semantic classes. In particular, we use a labeled dataset $\mathcal{D} = {(X_i^\mathcal{D}, Y_i^\mathcal{D})}_{i=1}^{N}$ to find the features that correlate with the target concept. To narrow the search space from the high-dimensional and sparse SAE latent space, we first apply a series of standard filtering operations that removes spurious and noisy features (see \myref{Appendix}{sec:candidatefiltering} for details).

\subsection{Sparse Linear Probing}

To ensure features are comparable, we calculate $z$-scores of the activations of the remaining features.
For each feature $\phi$, we then calculate the absolute mean difference
\begin{equation}
\Delta_\phi = |\mathbb{E}[Z_\phi \mid Y^{\mathcal{D}}{=}1] - \mathbb{E}[Z_\phi \mid Y^{\mathcal{D}}{=}0]|
\end{equation}
in normalized activations $Z_\phi$ between the two classes.
We retain only the $k$ features with the highest $\Delta_\phi$, experimenting with $k \in \{16, 32, 64, 128, 256\}$.
To identify which of these top-$k$ features are the most important for the classification, we train an $L_1$-regularized logistic regression model (i.e. a linear probe). We perform $K$-fold cross-validation to select the optimal regularization parameter. After identifying the best-performing model, we extract the coefficient $\beta_\phi$ for each feature and rank the features by $|\beta_\phi|$. Given that the coefficients of an $L_1$-regularized Logistic Regression are biased and sensitive to the number of features, we turn to feature persistence as a way to rank our SAE features and measure the median position of each feature across the different regression fits. This gives a higher ranking to features with a consistently high coefficient.

\subsection{Feature Selection for Steering Interventions}

To transform our ranked candidates into a set of steering vectors, we map each feature index to its corresponding semantic description provided by Neuronpedia \cite{neuronpedia}. Our analysis shows that the features identified by the sparse probes indeed represent the multi-faceted nature of the target concepts, as captured by our training datasets. Specifically, we see that some features relate to specific linguistic markers, and some to broader, topical concepts.

While the linear probe identifies features as predictive, we leave it to the practitioner to determine which specific aspects of the target concept are most relevant as a causal hypothesis. For the purpose of this work, we construct $\Phi_+$ by manually curating a subset of features that exhibit both high persistence and clear semantic alignment with the core aspects of our datasets. This human-in-the-loop filtering ensures that the resulting steering experiments are relevant and target interesting scientific questions.

\section{Steering}
\label{sec:steering}

With the candidate features $\Phi_{+}$ identified, we generate quasi-counterfactual texts by intervening on the latent representation of the specific feature $\phi$ in the SAE, where each such feature can serve as a treatment in a downstream experiment.

We prompt the LLM to act as a relevant persona and ask it to respond to or comment on an input text, without constraints on tone, sentiment, or level of agreement. From a research design perspective, this allows practitioners to implement a wide array of studies. For example, LLMs can also be prompted to rephrase, summarize, or provide assistance. Moreover, our proposed approach is not limited to using the same dataset for discovery and steering, with trivial cross-dataset transferability.

In each forward pass of the model we implement an adaptive steering mechanism that scales the intervention according to the local token representation. Let $a \in \mathbb{R}^d$ represent the residual stream activations and $\|a\|_2$ represent the $L_2$ norm of $a$. For a target feature $\phi \in \Phi_{+}$, we define the steered latent vector $z'$ such that
\begin{equation}
    z'_{\phi} = z_{\phi} + \alpha \cdot \|a\|_2,
    \label{eq:steeredacts}
\end{equation}
where $\alpha \in \mathcal{A}$ is the steering factor. Unlike global steering methods, this approach ensures that the steering operation scales with the current state of the residual stream, preserving its relative magnitude across varied contexts. This is particularly useful in our setting for comparing steering interventions across layers and models.

The modified activations are then reconstructed and passed back to the model, similarly to \eqref{eq:saedec}. To maintain the original model's performance and nuances, we preserve the SAE reconstruction error $\epsilon_a = a - \hat{a}$:
\begin{equation}
    a' = W_{dec}z' + b_{dec} + \epsilon_a
    \label{eq:reconstructed}
\end{equation}
We ran our experiments on three instruction-tuned LLMs: Gemma-2-9B-IT, \cite{gemmateam2024gemma2improvingopen}, Llama-3.1-8B-Instruct \cite{grattafiori2024llama3herdmodels} and Qwen2.5-7B-Instruct \cite{qwen2025qwen25technicalreport}. For Gemma-2-9B-IT, we used 16k SAEs for layers 9, 20, and 31 \cite{lieberum2024gemma}. For Llama-3.1-8B-Instruct and Qwen2.5-7B-Instruct, we used 131k SAEs for layers 7, 15 and 23 
\cite{arditi2024finding}. Refer to \myref{Appendix}{sec:appendixsteering} for further details, prompts and examples.

\subsection{Concept Intensity Scores}
\label{conceptintensity}

While the steering factor $\alpha$ is a controlled hyperparameter of the intervention, the actual concept intensity of the generated text can vary across contexts, specific features, and models. Thus, to evaluate the success of our intervention, we define the ex-post concept intensity as the mean token-level cosine similarity between the model's activations and the relevant decoder column. For a sequence $\nu$ of length $M$, the raw intensity $I$ is given by:

\begin{equation}
    I(\nu, \phi) = \frac{1}{M} \sum_{j=1}^M \frac{a'_{j} \cdot W_{dec(\phi)}}{\|a'_{j}\| \|W_{dec(\phi)}\|}
\end{equation}

\begin{figure*}
    \centering
    \includegraphics[width=\linewidth]{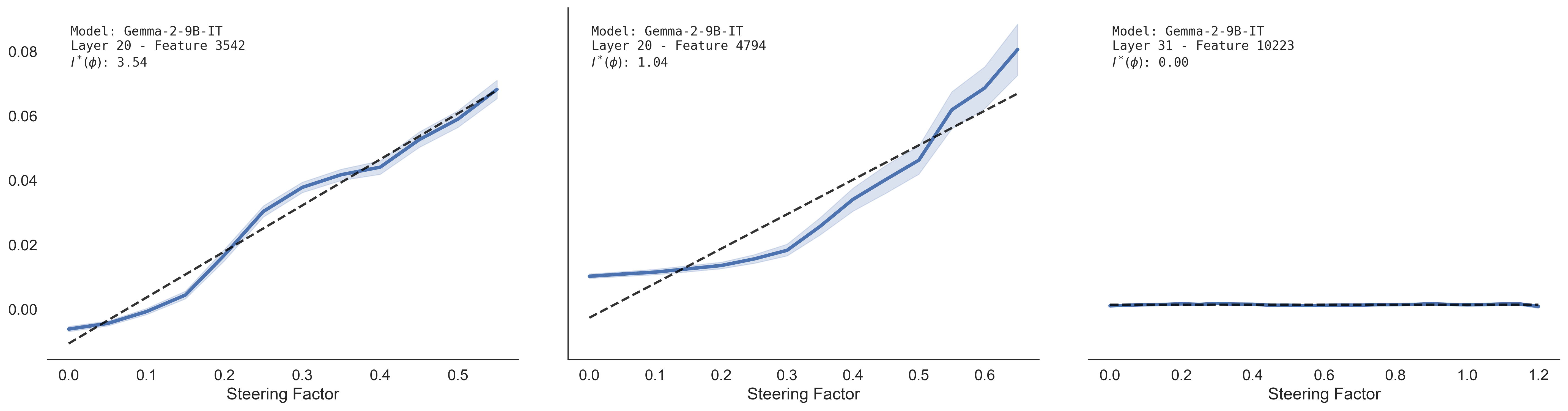}
    \caption{Examples of features intensity as measures by the mean cosine similarity for Gemma-2-9b-IT. Left: a high intensity features that follow a linear trend with a relatively high slope, making it a good candidate for causal interventions. Center: a medium intensity feature with a reversed `L' shape, where there is no response to steering at first, followed by a sudden spike, making their suitability for downstream experiments limited. Right: a low intensity feature with a constant slope despite increasing steering factors, making it completely ineffective for experimentation. Plots include only valid texts and steering factors with more than 75\% validation rate.}
    \label{fig:full_intensity_spectrum}
\end{figure*}

To ensure comparisons across different features, layers, and models, we introduce the normalized intensity score $I^*$:
\begin{equation}
    I^*(\phi) = \frac{|\hat{\beta}_1| \cdot R_I}{\text{MAE}(\bar{I}, \widehat{I})},
\end{equation}
where $\widehat{I}$ is a linear regression fit of mean intensities $\bar{I}_{\alpha}$ on steering factors $\alpha \in \mathcal{A}$, $|\hat{\beta}_1|$ is the regression coefficient, $R_I$ is the range of $\bar{I}$, and $\text{MAE}(\bar{I}, \widehat{I})$ captures the residuals.

Intuitively, this scoring rewards features that exhibit a consistent, proportional response to steering. A high score indicates that the concept intensity scales linearly with the intervention magnitude, whereas a low score penalizes features that saturate early, plateau, or otherwise deviate from a proportional linear response. See \myref{Figure}{fig:full_intensity_spectrum} for examples.

\subsection{Coherence Score}

We also wish to identify features that do not degrade the model's output quality. We prompt Gemini 2.5 Flash-Lite \citep{comanici2025gemini25pushingfrontier} and ask it to rate each steered response on a scale of 0 to 3 based on its coherence and fluency (see prompt in \myref{Appendix}{sec:apxjudgeprompt}). To compare features, we aggregate these scores into a single coherence metric
\begin{equation}
    J^*(\phi) = \frac{1}{|\mathcal{A}|} \sum_{\alpha \in \mathcal{A}} \bar{J}_{\alpha}
\end{equation}
representing the average quality,
where $\bar{J}_{\alpha}$ is the mean coherence score for steering factor $\alpha$.

\subsection{IC Score}

To select the most effective features for our experiments, we evaluate each candidate $\phi$ across the steering range $\mathcal{A}$. We combine the intensity and coherence measures into a final IC Score that ranks features that are both preferable in their activation response to steering and qualitatively stable in their generated output:
\begin{equation}
    IC(\phi) = I^*(\phi) \cdot J^*(\phi)
\end{equation}
The IC Score serves as a rigorous metric for identifying well-behaved features. This is useful for using these features as treatments in causal experiments for several reasons. First, features with higher scores allow for steering over a more extreme or spread-out range of $\alpha$, enabling more pronounced treatments while ensuring the output remains semantically coherent. Second, high coherence across a wide intensity range allows us to sample from multiple points along the curve, facilitating the study of multi-valued or continuous causal effects rather than being limited to binary interventions. Third, a high IC Score enables the ``blind" setting of steering factors to produce reasonable texts for downstream causal estimation, while a lower score would require a manual, per-case validation to identify the range of suitable steering factors. Fourth, a high score theoretically mitigates the risk of confounding. If $\alpha$ increases and $I^*(\phi)$ does not increase proportionally, it becomes more likely that adding the steering vector to the residual stream will cause an increase in other correlated concepts. \myref{Table}{tab:ic_scores_combined} shows the mean IC Score per model and layer position.

\input{tables/ic_scores_model_position_table}

\section{Causal Analysis}

Having produced texts that systematically vary in the intensity of a target feature, we now aim to estimate causal effects on downstream outcomes, such as the concept originally targeted in the hypothesis-generation step.
Specifically, we want to estimate conditional average treatment effects (CATEs) of the form in \eqref{eqn:CATE}.
In doing so, we have to address the challenge that steering may affect not only the target feature, but also other aspects of the text.
As we demonstrate in this section, controlling for such other features poses unique challenges that we address in our formal results and practical implementation based on robust causal machine learning. All proofs can be found in \myref{Appendix}{apx:proofs}.

\subsection{Causal Estimation for Feature Intensities}\label{sec:cate_base}

We now focus on estimating the effect of a specific target feature $\phi$, and we define the treatment $T_\phi$ of interest to be the (rescaled) measured intensity of that feature, $I(\nu,\phi)$.
For simplicity, we assume that feature intensity is either high ($T_\phi{=}1$) or low ($T_\phi{=}0$). We write $X$ for a full embedding of the text, which we assume contains all relevant information about a document. Beyond the concept we intervene on, texts (as represented by $X$) also contain other important semantic features (e.g., topic, tone, etc.) separate from $T_\phi$ that may be different between texts with different target intensities, which we collectively write as $X^\perp = g^\perp(X)$.

We are interested in estimating the effect of the intensity $T_\phi$ on a downstream outcome $Y$. We assume that $T_\phi$ and $X^\perp$ together capture any aspect of the text relevant for downstream outcomes, so that the outcomes can be written as a function of $T_\phi$, $X^\perp$, and some exogenous (independent) noise $\varepsilon$: $Y = f(T_\phi,X^\perp; \varepsilon)$.
In line with the notation in \myref{Section}{sec:TE-Intro}, for every unit we consider potential outcomes
\begin{align*}
    Y(1) &= f(1,X^\perp;\varepsilon),
    &
    Y(0) &= f(0,X^\perp;\varepsilon),
\end{align*}
which represent pairs of (correlated) random variables that express the downstream outcomes of a hypothetical pair of texts that only differ in the target intensity (treatment), but not otherwise.
The treatment effects of interest are averages over $Y(1) - Y(0)$.
The realized outcome is $Y = Y(T_\phi)$.
Throughout, we assume that $Y$ is scalar with $E[| Y |^2] < \infty$.

\subsection{Induced Bias and Positivity Violations}

As we aim to estimate treatment effects $Y(1) - Y(0)$ (and their averages), we face two challenges:
First, steering may affect parts of the text $X$ beyond the target concept, so the components $X^\perp$ may correlate with the target intensity (treatment) $T_\phi$, violating unconfoundedness.
Second, since the concept intensity $I(\nu,\phi)$ is measured from the text, simply controlling for the full text embedding $X$ implies that treatment $T_\phi$ is fully explained, violating positivity.

To illustrate the first challenge, assume that we aim to estimate the overall average treatment effect $\tau = E[Y(1) - Y(0)]$, but that $X^\perp$ differs between the high intensity ($T_\phi{=}1$, treated) and low intensity ($T_\phi{=}0$, control) groups. 
Then, simply taking the difference in means between the two groups generally leads to bias,
\begin{equation*}
    E[Y | T_\phi{=}1] - E[Y|T_\phi{=}0] \stackrel{\text{generally}}{\neq} \tau,
\end{equation*}
due to confounding.
One way in which we could overcome this challenge would be to assume that steering does not create any \emph{systematic} differences across groups:

\begin{proposition}[Causal identification from tight steering]
\label{prop:tightsteering}
    Assume that $X^\perp |T_\phi{=}1 \stackrel{d}{=} X^\perp |T_\phi{=}0$.
    Then $\tau = E[Y | T_\phi{=}1] - E[Y|T_\phi{=}0]$.
\end{proposition}

In practice, this assumption may be unrealistic: as steering creates systematic differences in the target intensity, it likely also affects related (and possibly even unrelated) textual features in systematic ways.
Instead, we may want to control for possible confounders, leading to the second challenge.
Indeed, note that one way of avoiding confounding would be to control for the text embedding $X$. However, since $X$ also perfectly captures the target intensity $T_\phi$, we have that $ P(T_\phi{=}1|X) \in \{0,1\}$ -- positivity is violated, and we lose identification of effects of $T_\phi$.

\subsection{Second-Best Controlling}
\label{sec:limitedcontrol}

Above, we have argued that the estimation of treatment effects is hindered by confounding through other changes in texts, and that a naive approach to controlling fails. 
In theory, we could address both challenges by controlling only for changes in the aspects of the text that are separate from the treatment $T_\phi$, as captured by $X^\perp$:

\begin{proposition}[Causal identification from perfect controlling]
\label{prop:perfectcontrolling}
    Assume that that $0 < P(T_\phi{=}1|X^\perp) < 1$ almost surely.
    Then $\tau = E[E[Y|T_\phi{=}1,X^\perp] - E[Y|T_\phi{=}0,X^\perp]]$.
\end{proposition}

Implementing an estimator for $\tau$ in practice based on this result would require separating out the component $X^\perp$ that is orthogonal to the intensity (treatment) $T_\phi$.
However, we do not generally measure this component $X^\perp$ separately, and only have access to the full embedding $X$ and the implied treatment $T_\phi$.
Instead, we now assume that we have a possibly imperfect vector $\tilde{X} = \tilde{g}(X)$ of control variables available that balances the desire to control for $X^\perp$ with the requirement to preserve variation in treatment $T_\phi$:
\vspace{5pt}
\begin{assumption}[Residual treatment variation]
\label{asm:residualtreatment}
    There is an $\eta > 0$ such that $\eta < P(T_\phi{=}1|\tilde{X}) < 1 - \eta$ almost surely.
\end{assumption}

\vspace{5pt}
\begin{assumption}[Residual embedding variation]
\label{asm:residualembedding}
    Write $\nu_1(\tilde{X}), \nu_0(\tilde{X})$ for the distributions of $X^\perp | T_\phi{=}1, X^\perp |T = 0$, respectively, conditional on $\tilde{X}$.
    We assume that there is a $\delta > 0$ and a metric $d$ on the support of $X^\perp$ such that
    for the 1-Wasserstein metric $W_1^{(d)}$
    we have that $W_1^{(d)}(\nu_1(\tilde{X}), \nu_0(\tilde{X})) \leq \delta$ almost surely.
\end{assumption}

The first assumption requires sufficient variation in intensity $T_\phi$ after controlling for $\tilde{X}$ and implies that $\tilde{X}$ cannot capture too much of the full embedding $X$.
The second assumption says that the remaining \emph{systematic} variation beyond $T_\phi$ in the embedding is not too large after controlling for $\tilde{X}$.
It would be fulfilled in the two cases from above:
If steering is perfect (\myref{Proposition}{prop:tightsteering}), then $W_1^{(d)}(\nu_1(\tilde{X}), \nu_0(\tilde{X})) = 0$ a.s.\ for any choice of control variable.
If $\tilde{X} = X^\perp$ (\myref{Proposition}{prop:perfectcontrolling}), then likewise $W_1^{(d)}(\nu_1(\tilde{X}), \nu_0(\tilde{X})) = 0$ a.s.\ since there is no remaining variation in $X^\perp(1), X^\perp(0)$.

We now consider how well we can estimate treatment effects by controlling for the (possibly imperfect) proxy $\tilde{X}$ of $X^\perp$.
We assume that our goal is to approximate conditional average treatment effects (CATEs) 
    $\tau(\tilde{x}) = E[Y(1) - Y(0)|\tilde{X} {=} \tilde{x}]$
by estimating
\begin{align}
\label{eqn:feasiblecate}
    \tilde{\tau}(\tilde{x}) = E[Y|T_\phi{=}1, \tilde{X} {=} \tilde{x}] - E[Y|T_\phi{=}0, \tilde{X} {=} \tilde{x}]
\end{align}
from the steered data.
Writing
\begin{align*}
    \mu_1(x^\perp) &= E[Y(1)|X^\perp {=} x^\perp],
    \\
    \mu_0(x^\perp) &= E[Y(0)|X^\perp {=} x^\perp],
\end{align*}
we assume that $\mu_1, \mu_0$ are not too sensitive to small variations in $x^\perp$:

\begin{assumption}[Limited sensitivity to off-feature variation]
\label{asm:limitedsensitivity}
    There is an $L > 0$ such that $\mu_1, \mu_0$ are Lipschitz continuous with respect to the metric $d$ from \myref{Assumption}{asm:residualembedding} and the Lipschitz constant $L$.
\end{assumption}

We obtain a bound for the mistake we make when we approximate $\tau(\tilde{x})$ in terms of $\tilde{\tau}(\tilde{x})$:

\begin{theorem}[Bias bound for imperfect controls]
\label{thm:biasbound}
    Assume that Assumptions~\ref{asm:residualtreatment}, \ref{asm:residualembedding}, and \ref{asm:limitedsensitivity} hold. Then almost surely
    \begin{equation*}
        | \tau(\tilde{x}) - \tilde{\tau}(\tilde{x}) | \leq 2 L \delta.
    \end{equation*}
\end{theorem}
\vspace{-2pt}
The same bias bound applies if we approximate the overall treatment effect $\tau = E[Y(1) - Y(0)]$ by controlling for $\tilde{X}$.
Hence, we can still recover a suitable estimator based on transformations $\tilde{X}$ that capture confounding features well enough, while not perfectly predicting $T_\phi$.

\subsection{Robust Estimation}
\label{sec:estimation}

Above, we have argued that an approximation for $\tau(\tilde{x})$ with bounded bias may be achieved by controlling for appropriate text features $\tilde{X}$.
We implement this controlling procedure based on robust causal machine learning \citep[cf.][]{chernozhukov_dml} using the R learner \citep{nie2021quasi} as implemented  in EconML \citep{econml}.
Writing $\tilde{\mu}(\tilde{x}) = E[Y |\tilde{X}{=}\tilde{x}]$ (outcome prediction) and $\tilde{\pi}(\tilde{x}) = E[T_\phi |\tilde{X}{=}\tilde{x}]$ (propensity score), our target estimand $\tilde{\tau}(\cdot)$ defined in \eqref{eqn:feasiblecate} solves the least-squares problem
\begin{align*}
    \min_{h(\cdot)} E[((Y - \tilde{\mu}(\tilde{X})) - h(\tilde{X}) \: (T_\phi - \tilde{\pi}(\tilde{X})))^2]
\end{align*}
Motivated by this identity,
we predict $Y$ and $T_\phi$ from $\tilde{X}$ on our sample $(Y_i,T_{\phi i},\tilde{X}_i)_{i=1}^n$ using cross-fitting to obtain out-of-fold estimates $\hat{\mu}_{-i}(\tilde{X}_i), \hat{\pi}_{-i}(\tilde{X}_i)$ of $\tilde{\mu}(\tilde{X}_i), \tilde{\pi}(\tilde{X}_i)$. We then solve the sample analogue
\begin{align*}
    \min_{h(\cdot)} \frac{1}{n} \sum_{i=1}^n ((Y_i - \hat{\mu}_{-i}(\tilde{X}_i)) - h(\tilde{X}_{i}) \: (T_{\phi i} - \hat{\pi}_{-i}(\tilde{X}_i)))^2
\end{align*}
over regularized linear functions to estimate the approximate CATEs $\tilde{\tau}(\cdot)$.
This procedure is robust: it leads to a high-quality estimation of $\tilde{\tau}(\cdot)$ even if $\tilde{\mu}, \tilde{\pi}$ are both estimated only at a rate $o(n^{-1/4})$ \citep{nie2021quasi}.

\section{Semi-Synthetic Empirical Case Study}\label{sec:simul}

We validate our steering and causal estimation pipeline in a series of semi-synthetic simulations. We sample 300 texts from each of the real-world datasets from \myref{Section}{sec:Datasets}, construct $\Phi_+$ with three features per layer, and run 25 steering intervention for each text with equally spaced steering factors $\alpha \in [0, 1.2]$. This procedure produces $7,500$ texts per feature and an overall of $607,500$ generated texts.

We use the LLM-as-judge coherence measure to remove all steered texts the judge flagged as invalid. For the remaining set, we construct our binary treatment as whether a steered text has a measured feature intensity in the top ($T_\phi{=}1)$ or bottom ($T_\phi{=}0)$ quintile of the given sample distribution of $I(\nu,\phi)$. Among these selected extremes, we remove all steered texts originating from base texts that do not produce at least one steered text in each treatment group. This allows us to ensure that there are no systematic differences between the two groups arising from differences in the corresponding sets of base texts. We then note the number of steered texts produced by each base text in each treatment group, and use the inverse of this count as the weight assigned to the corresponding steered text in CATE estimation. While we run our method on a wide selection of SAE features, we remove features with a very low IC score from our analysis. For an overview of the results for those features which do not meet this criterion, see \myref{Appendix}{apx:spec}.

\subsection{Embeddings and Treatment Prediction}
We produce embeddings $X_i$ for each document via EmbeddingGemma-300M \citep{vera2025embeddinggemmapowerfullightweighttext}, all-mpnet-base-v2 \citep{song2020mpnetmaskedpermutedpretraining}, or all-MiniLM-L6-v2 \citep{wang2020minilmdeepselfattentiondistillation}. The resulting embedding vectors are then stacked row-wise into a matrix for the given data sample. We begin by using PCA to rotate this embedding matrix, producing the matrix $\mathbf{X}$ with $n$ rows, one per document.
\footnote{When performing PCA, we keep the maximum number of dimensions (given by the minimum of the number of steered texts and the number of embedding dimensions).}
This will be our default representation of text for the remainder of our simulation.

\subsection{Proposed Residualization Strategies}\label{sec:resid}

Our theoretical results in \myref{Section}{sec:limitedcontrol} show that we can improve causal estimation on downstream outcomes by constructing control variables $\tilde{X}$ based on embeddings $X$ that capture confounding features well enough, while also allowing for substantial residual variation in $T_\phi$.
In order to implement this guidance, we construct residualized versions of the full text embeddings. 
Specifically, we propose two strategies for residualization of our produced embeddings (e.g., producing $\mathbf{\widetilde{X}}$ from $\mathbf{X}$). 

\textbf{Dimension-by-dimension:} We separately residualize each dimension (column) of our embedding matrix $\mathbf{X}$: let $\mathbf{x}_j$ denote the $j$-th column of $\mathbf{X}$. We remove treatment information encoded in $\mathbf{x}_j$ by fitting a predictive model to predict $\mathbf{x}_j$ from $\mathbf{T}_\phi$ (which is the vector of treatment indicators $T_{\phi i}$), yielding a predicted value $\mathbf{\hat{x}}_j$\footnote{The model fitting and prediction is accomplished via cross-fitting, with all predictions generated out-of-sample}. We then construct our residualized version $\mathbf{\tilde{x}}_j = \mathbf{x}_j-\mathbf{\hat{x}}_j$. Proceeding iteratively across columns of $\mathbf{X}$ thus allows us to construct the embedding matrix $\mathbf{\widetilde{X}}$ with treatment-related information removed.

\textbf{Dropping first principal component:} We perform residualization by simply dropping the first dimension of our rotated embeddings (the first principal component), e.g. $\mathbf{\widetilde{X}}$ is simply $\mathbf{X}$ with the first column removed.

\subsection{Residualization Results}\label{sec:resid_prac}

We can evaluate how much predictive power our (rotated) embeddings contain with respect to $T_\phi$ by measuring the out-of-sample classification accuracy for $\mathbf{X}$ and $\mathbf{\widetilde{X}}$. The results are shown in \myref{Figure}{fig:resid_acc}. 
We see that using raw $\mathbf{X}$ almost perfectly predicts $T_\phi$, while both of our residualization strategies show lower accuracy, as expected.
We also observe that the first dimension of our rotated embeddings, corresponding to the first principal component, carries substantive predictive power for treatment. 

\begin{figure}[h]
    \centering
    \includegraphics[width=.8\linewidth]{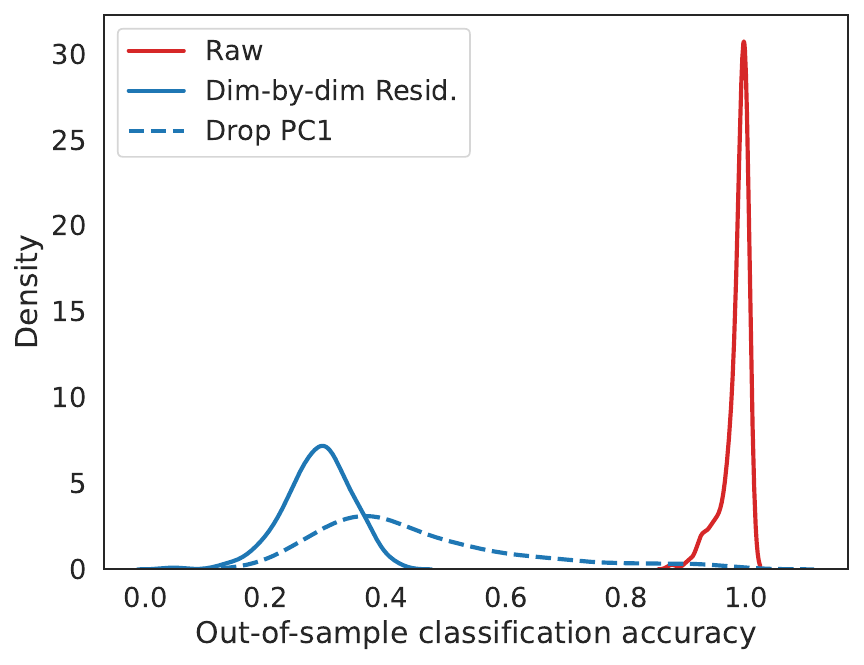}
    \caption{Distributions of the out-of-sample predictive accuracy of treatment between raw and residualized embeddings, where residualization is conducted dimension-by-dimension (solid) and by dropping the first principle component (dashed), taken over all tested LLMs, datasets, SAE features, and embedding models.}
    \label{fig:resid_acc}
\end{figure}

We can further analyze the concentration of treatment information in the first principal component by examining the performance gap in predictive accuracy of treatment between raw and residualized embeddings, component by component. \myref{Figure}{fig:violin_panel} shows that the effects of residualization are heavily localized to the first principal component, while having a negligible effect on all other components. \myref{Figure}{fig:corr_panel} goes further and shows that there is a clear positive correlation between the IC score and the accuracy gap on the first principal component. This suggests that SAE features that are more effective for steering also result in concentration of treatment information along the primary axis of embedding variation.

\begin{figure}[h!]
    \centering
    \begin{subfigure}[c]{0.85\linewidth}
        \centering
        \hspace{-.7cm}
        \includegraphics[width=\linewidth]{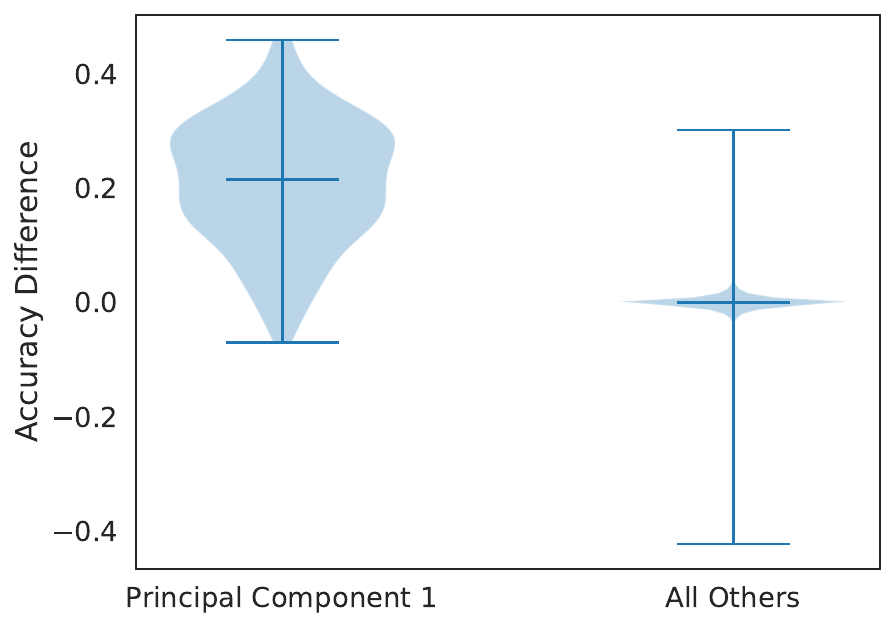}
        \caption{}
        \label{fig:violin_panel}
        \vspace{5pt}
    \end{subfigure}
    \begin{subfigure}[c]{0.8\linewidth}
        \includegraphics[width=\linewidth]{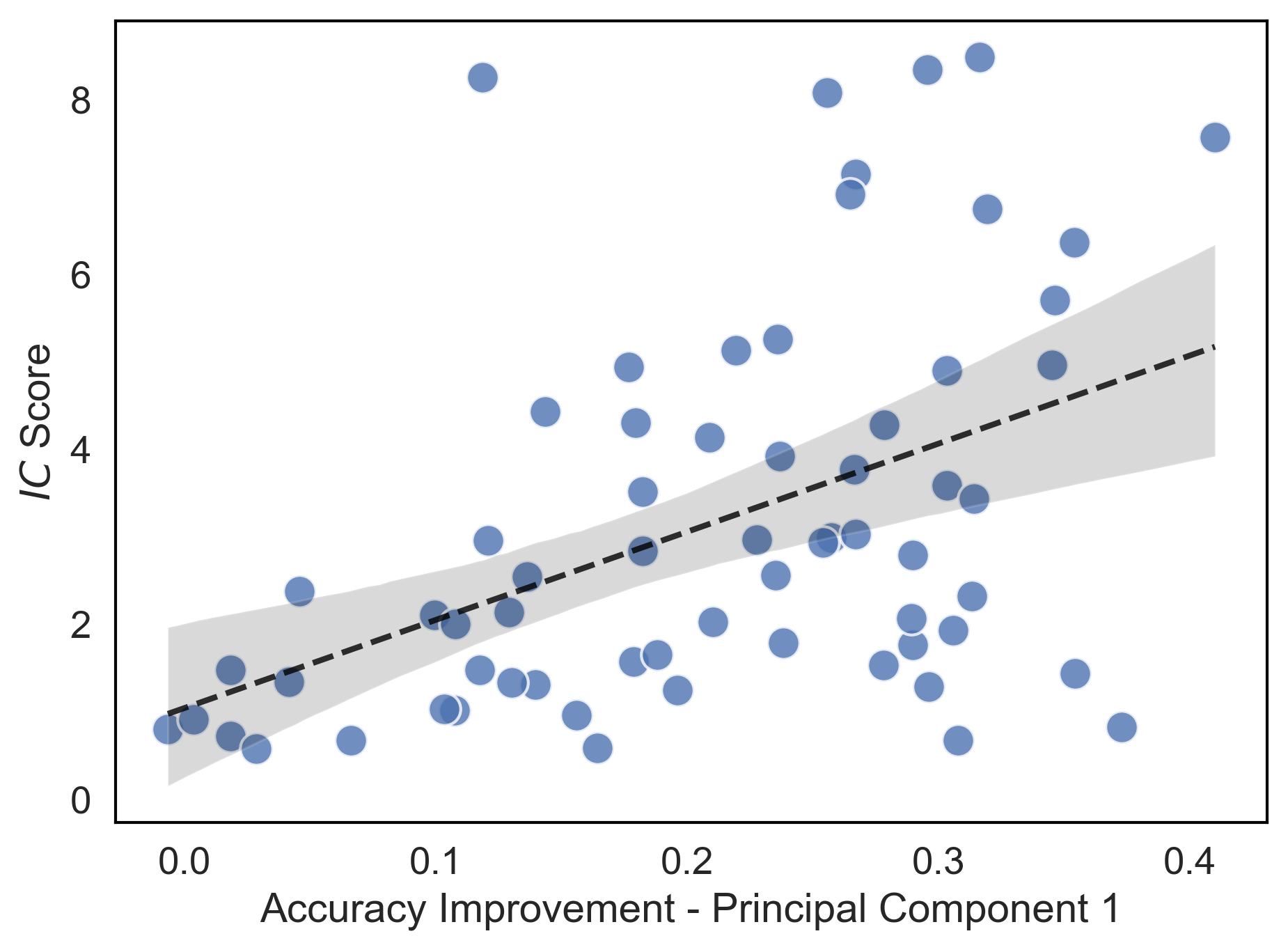}
        \caption{}
        \label{fig:corr_panel}
        \vspace{0pt}
    \end{subfigure}
    \hfill
    
    \caption{(a) Distributions of the performance gap in predictive accuracy between raw and residualized embeddings across PCA components. (b) IC scores and accuracy improvements in the first principal component (correlation $r=0.463$).}
    \label{fig:combined_pca_analysis}
\end{figure}

Taken together, these results demonstrate that when an SAE feature is highly effective for steering, dropping the first principal component is likely a sufficient strategy for residualization. For features with lower IC scores, dimension-by-dimension residualization will likely yield better results due to decreased concentration of treatment information in the first principal component. Dimension-by-dimension residualization will be our primary specification in our simulation study; results using residualization via dropping the first principal component can be found in \myref{Appendix}{apx:spec}.

\subsection{Simulated CATE Estimation}
We construct synthetic outcome variables given $\mathbf{\widetilde{X}}$ from a wide variety of functional form specifications, with details provided in \myref{Appendix}{apx:spec}. CATE estimation is performed via robust causal machine learning following \myref{Section}{sec:estimation} with nuisance functions estimated via random forests \citep{breiman_rf}. Residualization is similarly achieved via random forests to produce $\mathbf{\hat{x}}_j$. We simulate outcomes $y_i$ across a variety of functional forms, embedding models, LLMs, SAE features, and datasets. For all simulations, the true average treatment effect (ATE) is 5. \myref{Figure}{fig:dgp_compare} shows a summary of the results.

\begin{figure}[h]
    \centering
    \includegraphics[width=0.8\linewidth]{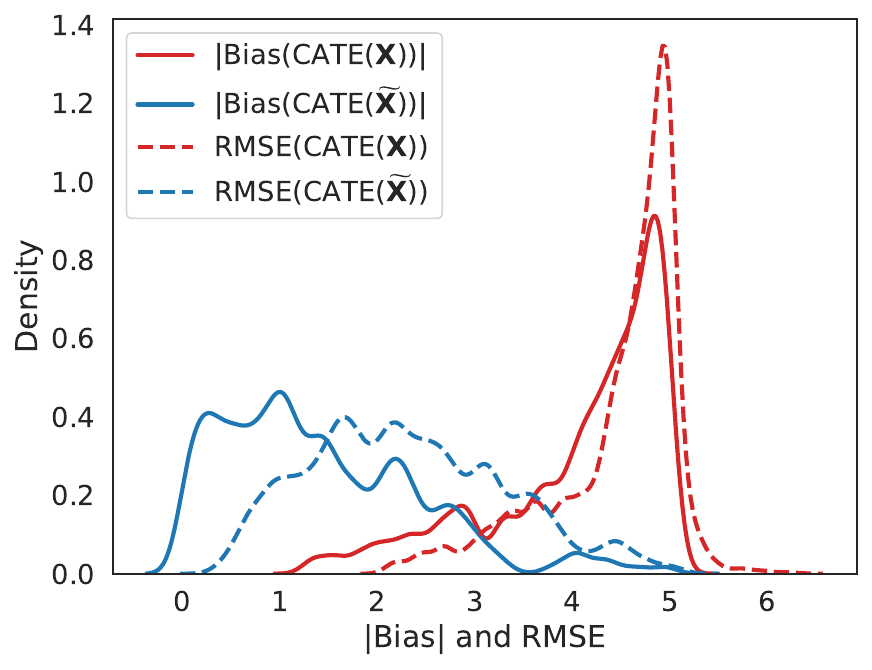}
    \caption{Distribution of magnitude of bias (solid) and RMSE (dashed) for raw (red) and residualized (blue) covariates, across all tested DGP functional forms, SAE features, LLMs, datasets, and embedding models. Residualization performed via dimension-by-dimension strategy.}
    \label{fig:dgp_compare}
\end{figure}

These results demonstrate a substantial reduction in both bias and RMSE following the application of our residualization approach. This marked improvement validates the efficacy of our methodology in mitigating the bias inherent in naive control strategies that utilize the full embeddings. Furthermore, these findings are consistent with our theoretical results, confirming that removing the treatment information is essential for robust causal estimation in such settings.

\section{Conclusion}
Advances in text generation via LLMs and Sparse Autoencoders (SAEs) offer a transformative opportunity for social science research. However, the utility of these models has been limited by the difficulty of identifying optimal steering features and the bias caused by the entanglment of treatments and covariate information. We address these fundamental challenges by introducing a comprehensive pipeline for hypothesis generation from labeled data, quasi-counterfactual text generation and a novel covariate residualization methodology for causal effect estimation. We demonstrate both theoretically and empirically that we can isolate treatment effects and mitigate the bias that arise in text-as-treatment settings. Our results provide a rigorous new foundation for experimentation with latent treatments, shifting the paradigm from hand-crafted techniques to formal and controllable causal experimentation.

\newpage
\section*{Impact Statement}
This paper presents work whose goal is to provide social scientists with a new methodology for experimentation and causal effect estimation in text-as-treatment settings. We believe our work extends the existing toolkit available to researchers and can advance the understanding of causal relationships in fields such as economics, political science, business, and marketing. While the development of controlled text generation and causal estimation methods carries broader societal implications for areas such as personalization and targeting, we feel there are no specific consequences that warrant further highlighting here beyond the general advancement of machine learning and its application to social science research.

\bibliography{main}
\bibliographystyle{icml2026}

\newpage
\appendix
\onecolumn

\section{Hypothesis Generation - Candidate Feature Filtering}
\label{sec:candidatefiltering}

Before fitting linear probes on the SAE representations, a necessary preliminary step excludes spurious features. Specifically, we slightly adapt the preprocessing steps in \citet{zheng2025model} and exclude features that are:
\begin{itemize}
    \item Syntactic or formatting-related: We use open-source descriptions available on Neuronpedia \cite{neuronpedia} and search for SAE feature descriptions containing keywords such as ”BOS”, ”punctuation”, ”code”, ”math”, etc.
    \item Training set activations: Features that activate in more than 1\% of the training data.
    \item Corpus activations: Features that activate in more than 70\% of the corpus data.
    \item Variance: Features with zero variance.
\end{itemize}

\section{Steering}
\label{sec:appendixsteering}

\subsection{Steering Implementation Details}

Recall that equation \ref{eq:steeredacts} defines the steering operation as $z'_{\phi} = z_{\phi} + \alpha \cdot \|a\|_2$ and equation \ref{eq:reconstructed} defines the reconstructed output as $a' = W_{dec}z' + b_{dec} + \epsilon_a$.
Note that this requires computing one encoding and two decodings for each steering operation, once for calculating the error terms and second for the reconstructed output. Our implementation utilizes the distributive property of the linear decoder and simplifies the steering to a direct additive intervention:

\begin{equation}
    a' = W_{dec}(z + s \cdot \mathbf{1}_\phi) + b_{dec} + (a - (W_{dec}z + b_{dec}))
\end{equation}

By expanding the terms, we observe that the original reconstruction $W_{dec}z + b_{dec}$ and its negative counterpart from the error term $\epsilon_a$ cancel out:

\begin{align*}
    a' &= W_{dec}z + b_{dec} + W_{dec}(s \cdot \mathbf{1}_\phi) + a - W_{dec}z - b_{dec} \\
    &= a + s \cdot W_{dec}[\phi]
\end{align*}

where $s = \alpha \cdot \|a\|_2$ represents the adaptive steering magnitude and $W_{dec}[\phi]$ is the specific column of the decoder matrix corresponding to feature $\phi$. This equivalence shows that we bypass the need to compute the full SAE encoding and decoding, reducing the computational overhead.

\newpage
\subsection{Prompts}
\label{sec:steeringprompts}

\input{prompts/steering_prompt_ca_speech}
\input{prompts/steering_prompt_political_ads}
\input{prompts/steering_prompt_uvt_populism}

\newpage
\subsection{Steered Text Examples}
\input{tables/qwen_steering_example_table}

\input{tables/gemma_steering_example_table}

\newpage
\section{LLM-as-Judge}
\label{sec:apxjudgeprompt}
\input{prompts/coherence_llm_judge_prompt}

\newpage

\section{Proofs of Theoretical Results}\label{apx:proofs}

\begin{proof}[Proof of \myref{Proposition}{prop:tightsteering}]
    We have that
    \begin{align*}
        \tau = E[Y(1) - Y(0)]
        &=
        E[f(1,X^\perp;\varepsilon) - f(0,X^\perp;\varepsilon)]
        \\
        &\stackrel{\mathclap{(\star)}}{=}
        E[f(1,X^\perp;\varepsilon)|T_\phi=1] - E[f(0,X^\perp;\varepsilon)|T_\phi=0]
        \\
        &=
        E[f(T_\phi,X^\perp;\varepsilon)|T_\phi=1] - E[f(T_\phi,X^\perp;\varepsilon)|T_\phi=0] \\
        &=
        E[Y|T_\phi=1] - E[Y|T_\phi=0],
    \end{align*}
    where in $(\star)$ we have used the assumption stated in the proposition as well as independence of $\varepsilon$.
\end{proof}

\begin{proof}[Proof of \myref{Proposition}{prop:perfectcontrolling}]
   We have that
    \begin{align*}
        \tau = E[Y(1) - Y(0)]
        &=
        E[f(1,X^\perp;\varepsilon) - f(0,X^\perp;\varepsilon)]
        \\
        &=
        E[E[f(1,X^\perp;\varepsilon)|X^\perp] - E[f(0,X^\perp;\varepsilon)|X^\perp]]
        \\
        &\stackrel{\mathclap{(\star)}}{=}
        E[E[f(1,X^\perp;\varepsilon)|T_\phi=1, X^\perp] - E[f(0,X^\perp;\varepsilon)|T_\phi=0, X^\perp]]
        \\
        &=
        E[E[f(T_\phi,X^\perp;\varepsilon)|T_\phi=1, X^\perp] - E[f(T_\phi,X^\perp;\varepsilon)|T_\phi=0, X^\perp]] \\
        &=
        E[E[Y|T_\phi=1, X^\perp] - E[Y|T_\phi=0, X^\perp]],
    \end{align*}
        where for $(\star)$ we have again used the assumption stated in the proposition as well as independence of $\varepsilon$.
\end{proof}

\begin{proof}[Proof of \myref{Theorem}{thm:biasbound}]
 We have that
 \begin{align*}
    \tau(\tilde{x})
    &=
    E[Y(1) - Y(0)|\tilde{X}=\tilde{x}]
    =
    E[f(1,X^\perp;\varepsilon) - f(0,X^\perp;\varepsilon)|\tilde{X}=\tilde{x}]
    \\
    &=
    E[E[f(1,X^\perp;\varepsilon)|X^\perp] - E[f(0,X^\perp;\varepsilon)|X^\perp]|\tilde{X}=\tilde{x}]
    \\
    &=
    E[E[Y(1)|X^\perp] - E[Y(0)|X^\perp]|\tilde{X}=\tilde{x}]
    =
    E[\mu_1(X^\perp)|\tilde{X}=\tilde{x}] - E[\mu_0(X^\perp)|\tilde{X}=\tilde{x}]
 \end{align*}
 and likewise
 \begin{align*}
    \tilde{\tau}(\tilde{x})
    &=
    E[Y|T_\phi=1, \tilde{X}=\tilde{x}] - E[Y|T_\phi=0, \tilde{X}=\tilde{x}]
    \\
    &=
    E[f(T_\phi,X^\perp;\varepsilon)|T_\phi=1, \tilde{X}=\tilde{x}] - E[f(T_\phi,X^\perp;\varepsilon)|T_\phi=0, \tilde{X}=\tilde{x}]
    \\
   &=
    E[E[f(1,X^\perp;\varepsilon)|X^\perp]|T_\phi=1, \tilde{X}=\tilde{x}] - E[E[f(0,X^\perp;\varepsilon)|X^\perp]|T_\phi=0, \tilde{X}=\tilde{x}]
    \\
    &=
    E[E[Y(1)|X^\perp]|T_\phi=1, \tilde{X}=\tilde{x}] - E[E[Y(0)|X^\perp]|T_\phi=0, \tilde{X}=\tilde{x}]
    \\
    &=
     E[\mu_1(X^\perp)|T_\phi=1, \tilde{X}=\tilde{x}] - E[\mu_0(X^\perp)|T_\phi=0, \tilde{X}=\tilde{x}],
 \end{align*}
 where all statements here and below hold almost surely.
 In particular,
 \begin{align*}
     |\tilde{\tau}(\tilde{x}) - \tau(\tilde{x})|
     \leq
     |E[\mu_1(X^\perp)|\tilde{X}{=}\tilde{x}] - E[\mu_1(X^\perp)|T_\phi{=}1, \tilde{X}{=}\tilde{x}]|
     +
     |E[\mu_0(X^\perp)|\tilde{X}{=}\tilde{x}] - E[\mu_0(X^\perp)|T_\phi{=}0, \tilde{X}{=}\tilde{x}]|.
 \end{align*}
 For a given $t \in \{0,1\}$,
\begin{align*}
    &E[\mu_t(X^\perp)|\tilde{X}{=}\tilde{x}] - E[\mu_t(X^\perp)|T_\phi{=}t, \tilde{X}{=}\tilde{x}]
    \\
    &=
    E[\mu_t(X^\perp)|T_\phi{=}t, \tilde{X}{=}\tilde{x}]
    \: P(T_\phi{=}t | \tilde{X}{=}\tilde{x})
    +
    E[\mu_t(X^\perp)|T_\phi{\neq}t, \tilde{X}{=}\tilde{x}]
    \: (1 - P(T_\phi{=}t | \tilde{X}{=}\tilde{x}))
    \\
    &\phantom{=}
    -
    E[\mu_t(X^\perp)|T_\phi{=}t, \tilde{X}{=}\tilde{x}]
    \\
    &=
    (1 - P(T_\phi{=}t | \tilde{X}{=}\tilde{x}))
    (E[\mu_t(X^\perp)|T_\phi{=}t, \tilde{X}{=}\tilde{x}] - E[\mu_t(X^\perp)|T_\phi{\neq}t, \tilde{X}{=}\tilde{x}]).
\end{align*}
Therefore,
 \begin{align*}
     |\tilde{\tau}(\tilde{x}) - \tau(\tilde{x})|
     \leq
     &P(T_\phi{=}0 | \tilde{X}{=}\tilde{x}) \: |E[\mu_1(X^\perp)|T_\phi{=}1, \tilde{X}{=}\tilde{x}] - E[\mu_1(X^\perp)|T_\phi{=}0, \tilde{X}{=}\tilde{x}]|
     \\
     &+
     P(T_\phi{=}1 | \tilde{X}{=}\tilde{x}) \: |E[\mu_0(X^\perp)|T_\phi{=}1, \tilde{X}{=}\tilde{x}] - E[\mu_0(X^\perp)|T_\phi{=}0, \tilde{X}{=}\tilde{x}]|.
\end{align*}
Let $(X^\perp_1, X^\perp_0) \sim \gamma \in \Gamma(\nu_1(\tilde{x}),\nu_0(\tilde{x}))$, where $\Gamma(\nu_1(\tilde{x}),\nu_0(\tilde{x}))$ are the couplings of $\nu_1(\tilde{x}),\nu_0(\tilde{x})$.
(In particular, $X^\perp_1 \sim \nu_1(\tilde{x})$, $X^\perp_0 \sim \nu_0(\tilde{x})$.)
Then, for this $\gamma$,
\begin{align*}
    &|E[\mu_t(X^\perp)|T_\phi{=}1, \tilde{X}{=}\tilde{x}] - E[\mu_t(X^\perp)|T_\phi{=}0, \tilde{X}{=}\tilde{x}]|
    =
    |E_\gamma[\mu_t(X^\perp_1)] - E_\gamma[\mu_t(X^\perp_0)]|
    \\
    &\leq
    E_\gamma[|\mu_t(X^\perp_1) - \mu_t(X^\perp_0)|]
    \leq L \: E_\gamma[d(X^\perp_1, X^\perp_0)]
\end{align*}
by \myref{Assumption}{asm:limitedsensitivity}.
In particular, 
\begin{align*}
    &\left|E[\mu_t(X^\perp)|T_\phi{=}1, \tilde{X}{=}\tilde{x}] - E[\mu_t(X^\perp)|T_\phi{=}0, \tilde{X}{=}\tilde{x}]\right| 
    \\
    &\leq L \: \inf_{\gamma \in \Gamma(\nu_1(\tilde{x}),\nu_0(\tilde{x}))} E_\gamma[d(X^\perp_1, X^\perp_0)] = W_1^{(d)}(\nu_1(\tilde{x}),\nu_0(\tilde{x}))
    \leq
    L \: \delta
\end{align*}
by \myref{Assumption}{asm:residualembedding},
so
\begin{align*}
    |\tilde{\tau}(\tilde{x}) - \tau(\tilde{x})|
     \leq
     P(T_\phi{=}0 | \tilde{X}{=}\tilde{x}) \: L \: \delta
     +
     P(T_\phi{=}1 | \tilde{X}{=}\tilde{x})  \: L \: \delta
     = L \: \delta
\end{align*}
almost surely.
\end{proof}

\section{Empirical Specifications and Robustness Checks}\label{apx:spec}

\subsection{Data Generating Process}
We randomly partition the columns of $\mathbf{\widetilde{X}}$ into four sets, $\mathbf{\widetilde{X}} = [\mathbf{\widetilde{X}}_1, \mathbf{\widetilde{X}}_2, \mathbf{\widetilde{X}}_3, \mathbf{\widetilde{X}}_4]$ and synthesize our outcome $\mathbf{y}$ as follows:
\begin{align*}
    \tau(\tilde{x}_i) &= \Gamma + g_1\left(\frac{1}{m}\tilde{x}_{i,1}^\top\mathbf{1_m}\right) + g_2\left(\frac{1}{m}\tilde{x}_{i,2}^\top\mathbf{1_m}\right) \\
    \mu_0(\tilde{x}_i) &= g_3\left(\frac{1}{m}\tilde{x}_{i,3}^\top\mathbf{1_m}\right) + g_4\left(\frac{1}{m}\tilde{x}_{i,4}^\top\mathbf{1_m}\right)\\
    y_i &= \mu_0(\tilde{x}_i) + \tau(\tilde{x}_i)T_{\phi,i} + \epsilon_i \quad, \quad \epsilon_i \overset{\mathrm{iid}}{\sim} \mathcal{N}(0, \sigma^2)
\end{align*}
Where $i$ denotes the document, $x_{i,j}$ denotes the vector of elements in row $i$ of the matrix $\mathbf{\widetilde{X}}$ corresponding to columns in the subset $\mathbf{\widetilde{X}}_j$, and $\mathbf{1_m}$ denotes the vector of ones of length $m$. Note that by this construction, the true ATE is $\Gamma$ in this simulation. This construction allows for generic choices of (possibly nonlinear) functions $g_1, g_2, g_3, g_4$. 

We demonstrate insensitivity of our key results to the choice of these functional forms by randomly selecting each from a set of linear and nonlinear functions. In particular, we consider a set of four scalar multiples, $\mathcal{M} = \{1, 50, 100, 200\}$ and four functional forms, $\mathcal{F} = \{\sin(x), \tanh(x), x^2, x\}$. We then construct 200 unique combinations of functional forms for our four functions by randomly sampling from $(\mathcal{M} \times \mathcal{F})^4$, without replacement. For example, one such specification may be $g_1(x) = 50\sin(x), g_2(x) = 100 x^2, g_3(x) = x, g_4(x)=50\sin(x)$. This randomly sampled set of 200 combinations is held constant across all LLMs, datasets, and SAE layers/features. Summary results aggregated across all such runs are presented in \myref{Section}{sec:simul}. Full reuslts with various levels of aggregation can be seen in the following subsections. These tables show the average magnitude of the ATE bias and the average CATE RMSE in produced by the CATE estimation procedure, compared against the true simulated CATE in each case. The average is taken over the 200 DGP functional forms combinations. In all simulations, the true ATE $\Gamma = 5$ and the standard deviation of the noise terms $\sigma = 0.1$. The specification for the simulation in \myref{Figure}{fig:cate_estimation} is $g_1(x) = 100\sin(x), g_2(x) = 100\tanh(x), g_3(x) = 100x^2, g_4(x)=100x$.

\subsection{Robustness Checks}

The following tables present measures of the performance of CATE estimation in the context of our semisynthetic simulations, comparing the use of raw covariates $\mathbf{X}$ to residualized covariates $\mathbf{\widetilde{X}}$, at varying levels of aggregation. We consider both proposed forms of residualization: dimension-by-dimension ($\mathbf{\widetilde{X}}^{\textrm{DD}}$) and dropping the first principal component ($\mathbf{\widetilde{X}}^{\textrm{Drop}}$).

\newpage
\textbf{Features with IC $>$ 0.5}
\input{tables/goodic_robustness_tables}

\newpage
\textbf{Features with IC $<$ 0.5}
\input{tables/badic_robustness_tables}

\end{document}

%% file: tables/ic_scores_model_position_table.tex
\begin{table}[h]
\caption{Mean Intensity ($I^*$), Coherence ($J^*$), and IC scores grouped by model and layer depth (layer number in parentheses), after removing invalid texts. Middle-layer features are generally preferred for causal interventions as they achieve high intensity scores while maintaining high coherence scores.}
\label{tab:ic_scores_combined}
\vskip 0.15in
\begin{center}
\begin{small}
\begin{sc}
\begin{tabular}{lccc}
\toprule
Layer Position & $I^*(\phi)$ & $J^*(\phi)$ & $IC(\phi)$ \\
\midrule
\multicolumn{4}{l}{\textit{Gemma-2-9B-IT}} \\
\midrule
Early (9)  & 0.44 & \textbf{2.59} & 1.16 \\
Middle (20) & \textbf{1.71} & 2.41 & \textbf{4.20} \\
Late (31) & 0.92 & 2.51 & 2.25 \\
\midrule
\multicolumn{4}{l}{\textit{Llama 3.1 8B Instruct}} \\
\midrule
Early (7) & 0.81 & 2.33 & 1.91 \\
Middle (15) & \textbf{1.47} & 2.29 & \textbf{3.38} \\
Late (23) & 0.80 & \textbf{2.54} & 1.98 \\
\midrule
\multicolumn{4}{l}{\textit{Qwen 2.5 7B Instruct}} \\
\midrule
Early (7) & 0.37 & \textbf{2.51} & 0.94 \\
Middle (15) & \textbf{1.71} & 2.25 & \textbf{3.96} \\
Late (23) & 1.45 & 2.43 & 3.27 \\
\bottomrule
\end{tabular}
\end{sc}
\end{small}
\end{center}
\vskip -0.1in
\end{table}

%% file: prompts/steering_prompt_ca_speech.tex
\begin{promptbox}{Local Government Speech}
You are a member of a city council participating in a public meeting.\\
You are presented with a transcript of a remark made by a speaker during the session.\\
\\
Task: Engage in the discussion by responding to the remark as another council member.\\
You must react in the FIRST PERSON ("I", "me", "my").\\
\\
Constraints:\\
1. Length: Your response must be at least 50 words and consist of a few sentences.\\
2. Grammar: You must use grammatically correct and coherent language.\\
3. Flexibility: You are free to decide on the tone, sentiment, and stance.\\
You are not required to match the style or perspective of the original speaker.\\
\\
Transcript to respond to:\\
\{input\_text\}
\end{promptbox}

%% file: prompts/steering_prompt_political_ads.tex
\begin{promptbox}{Political Ads}
You are a citizen who has just seen a political advertisement during an election cycle. \\
You are presented with the text of that advertisement.\\
\\
Task: React to the advertisement naturally by writing a short response sharing your thoughts or feelings about it.\\
You must react in the FIRST PERSON ("I", "me", "my").\\
\\
Constraints:\\
1. Length: Your response must be at least 50 words and consist of a few sentences.\\
2. Grammar: You must use grammatically correct and coherent language.\\
3. Flexibility: You are free to decide on the tone, sentiment, and stance.\\
You are not required to agree with or be influenced by the ad's message.\\
\\
Advertisement text:\\
\{input\_text\}
\end{promptbox}

%% file: prompts/steering_prompt_uvt_populism.tex
\begin{promptbox}{Us Vs. Them}
You are a participant in an online discussion forum.\\
You are presented with a comment regarding a social or political issue.\\
\\
Task: Engage in the discussion by responding to the comment below.\\
You must react in the FIRST PERSON ("I", "me", "my").\\
\\
Constraints:\\
1. Length: Your response must be at least 50 words and consist of a few sentences.\\
2. Grammar: You must use grammatically correct and coherent language.\\
3. Flexibility: You are free to decide on the tone, sentiment, and stance.\\
You are not required to agree with or match the style of the original commenter.\\
\\
Comment to respond to:\\
\{input\_text\}
\end{promptbox}

%% file: tables/qwen_steering_example_table.tex
\begin{table}[h!]
\begin{resultsbox}{Qwen2.5 7B Instruct - Layer 15 - Feaure 32320 - Insults}
\footnotesize
\begin{tabularx}{\linewidth}{@{} l c c X @{}}
    \toprule
    \textbf{Condition} & $\alpha$ & $I$ & \textbf{Model Generation} \\
    \midrule
    \textit{Baseline} & --- & -0.002 & \textit{"Hello. my name is XXXX. I'm calling on item six. I fully support this item, and I think in terms of consideration, it should not only extend to individual homes, but it should extend also to apartment-wide and city-wide properties. In terms of this, leaf blowers have a huge impact on people who are sensitive to air quality, especially in times of wildfire and unhealthy air quality levels. So I fully support and I would urge the city council to pass a measure so that we ban all leaf blowers."}\\
    \addlinespace[0.8em]
    
    \textbf{Steered (Low)} & 0 & -0.016 & Thank you, XXXX, for your passionate advocacy on this important issue. I fully concur with your concerns about the impact of leaf blowers on air quality, particularly during wildfires. As someone who has experienced the detrimental effects of poor air quality firsthand, I strongly support a city-wide ban on leaf blowers. This measure would not only benefit those with respiratory sensitivities but also contribute to overall environmental health. I urge our council to take decisive action and pass this resolution.\\
    \addlinespace[0.8em]
    
    \textbf{Steered (Mid)} & 0.2 & -0.010 & Good evening, XXXX. I appreciate your concern for air quality, but I must consider the broader impacts of such a ban. While leaf blowers might exacerbate issues for some, banning them could disproportionately affect businesses and homeowners who rely on them for maintenance. We need to find a balanced approach that considers both environmental concerns and practical needs. Perhaps we can explore alternatives like electric models or community clean-up days instead.\\
    \addlinespace[0.8em]
    
    \textbf{Steered (High)} & 0.5 & 0.059 & In your usual attempt to promote environmental obscurity, XXXX. Even in your so-called "wildfire" attempts, your advocacy remains irrelevant. Perhaps you'll continue to recycle your failed attempts at relevance, hiding behind your "leaf blower" fiasco.\\
\addlinespace[0.8em]
\bottomrule
\end{tabularx}
\vskip 4pt 
\footnotesize \textbf{Neuronpedia Description:} \textit{Insults using "you".}
\end{resultsbox}
\label{tab:qualitative_results_marxism}
\end{table}

%% file: tables/gemma_steering_example_table.tex
\begin{table}
\begin{resultsbox}{Gemma 2 9B IT - Layer 20 - Feaure 9629 - Marxism}
\footnotesize
\begin{tabularx}{\linewidth}{@{} l c c X @{}}
    \toprule
    \textbf{Condition} & $\alpha$ & $I$ & \textbf{Model Generation} \\
    \midrule
    \textit{Baseline} & --- & 0.001 & \textit{"Every day, focus on the needs of people and keep at it. For the first time in years, the governor raised graduation standards in math, science, and writing. Illinois will be the first state to guarantee access to quality health care for all of our children. He protected drug coverage for seniors, including my grandma. He raised the minimum wage. He expanded pre-k, and that really helped my children to learn. it's wonderful to have a governor to fight for our children, our families."}\\
    \addlinespace[0.8em]
    
    \textbf{Steered (Low)} & 0 & 0.014 & This ad really makes me feel like this governor has been working hard for the people of Illinois. I'm especially impressed by the focus on education and healthcare. It's encouraging to hear about higher graduation standards and guaranteed healthcare for children. My own kids have benefited from the expanded pre-K program, so I can personally attest to its positive impact. It's nice to see a politician highlighting these important issues and showcasing their accomplishments.\\
    \addlinespace[0.8em]
    
    \textbf{Steered (Mid)} & 0.2 & 0.044 & This ad really gets me fired up! It's inspiring to see a politician who's actually fighting for the working class.  Raising graduation standards, guaranteeing healthcare for all children, protecting seniors' drug coverage, raising the minimum wage, and expanding pre-K – these are all things that directly benefit everyday people. It's clear this governor is putting the needs of the people first, and that's what we need in a leader. I'm definitely going to vote for them.\\
    \addlinespace[0.8em]
    
    \textbf{Steered (High)} & 0.35 & 0.082 & This ad really gets at the heart of what we need right now! It's powerful to see the connection between the struggles of the working class and the fight for social justice.  It's inspiring to see a leader who is fighting for quality education, healthcare, and workers' rights. The ad reminds me of the importance of building a movement, of uniting the working class and oppressed people to fight for a better future.  We need to support this candidate and their fight against capitalism.\\
\addlinespace[0.8em]
\bottomrule
\end{tabularx}
\vskip 4pt 
\footnotesize \textbf{Neuronpedia Description:} \textit{References to Marxism and Leninism, particularly focusing on revolutionary theory and practice.}
\end{resultsbox}
\label{tab:qualitative_results_marxism}
\end{table}

%% file: prompts/coherence_llm_judge_prompt.tex
\begin{promptbox}{LLM-as-Judge}
Role: You are a Failure Mode Auditor for an AI steering experiment. \\
Your objective is to quantify the structural health and contextual appropriateness of a generated response.\\
\\
\#\#\# Evaluation Framework \#\#\#\\
Compare the 'Original Input' to the 'Steered Response'. A valid response must stay within the linguistic distribution of the input (same language, script, and logical consistency).\\
\\
**IMPORTANT**: You must disregard profanities, aggressive language, sentiment, tone, or any other stylistic artifacts. Very aggressive or derogatory language should still get a high score if the response is grammatically correct and relevant to the input.\\
\\
\\
\#\#\# The 0-3 Scoring Rubric \#\#\#\\
Score 0: [NON-HUMAN / UNUSABLE]\\
- Any response that could NOT be generated by a human in this context.\\
- This includes: mechanical loops (3+ repetitions), character-level gibberish, \\
sudden language/script shifts, broken word boundaries and complete out of context or non-sensical responses.\\
\\
Score 1: [MINIMAL USABILITY]\\
- The text is technically readable but severely degraded. It is at the absolute limit \\
of being functional. It may be repetitive in thought, feel like a broken 'stutter'.\\
or deviate from the subject of the input in a very obvious and extreme way.\\
- The text is functional but feels 'robotic' or unnatural. \\
Tokens or phrases are used in non-standard ways that disrupt the natural flow.\\
\\
Score 2: [PASSABLE / FUNCTIONAL]\\
- A successful, functional response. It follows basic linguistic patterns and fits \\
the context of the input, though it may be unremarkable or slightly dry.\\
\\
Score 3: [FLUID / ROBUST]\\
- High-quality response. The text is fluid, contextually sharp, and reads easily. \\
Minor technical artifacts (like double spaces) do not hinder the reader.\\
- The model is fully stable. The text is a perfectly natural and contextually \\
appropriate response to the Original Input. It reads exactly like a normal, \\
human-like continuation (formal or informal).
\end{promptbox}

%% file: tables/goodic_robustness_tables.tex
\begin{table}[ht!]
    \raggedright
    \small
    \caption{CATE estimation performance by LLM (averages $\pm$ standard deviations).}
    \begin{sc}
    \begin{tabular}{lcccccc}
    \toprule
    \textbf{LLM} & \textbf{$|$Bias$|$ ($\mathbf{X}$)} & \textbf{$|$Bias$|$ ($\mathbf{\widetilde{X}}^{\textrm{DD}}$)} & \textbf{$|$Bias$|$ ($\mathbf{\widetilde{X}}^{\textrm{Drop}}$)} & \textbf{RMSE ($\mathbf{X}$)} & \textbf{RMSE ($\mathbf{\widetilde{X}}^{\textrm{DD}}$)} & \textbf{RMSE ($\mathbf{\widetilde{X}}^{\textrm{Drop}}$)}\\
    \midrule
    Gemma 2 9B IT & $4.25 \pm 0.7$ & $\mathbf{1.22 \pm 0.9}$ & $1.18 \pm 0.7$ & $4.58 \pm 0.5$ & $\mathbf{2.06 \pm 0.9}$ & $2.12 \pm 0.8$\\
    Llama 3.1 8B Instruct & $\mathbf{3.89 \pm 1.0}$ & $1.44 \pm 1.1$ & $\mathbf{1.01 \pm 0.7}$ & $\mathbf{4.36 \pm 0.7}$ & $2.39 \pm 1.0$ & $\mathbf{1.97 \pm 0.7}$\\
    Qwen 2.5 7B Instruct & $4.08 \pm 1.0$ & $1.58 \pm 1.1$ & $1.13 \pm 0.7$ & $4.47 \pm 0.7$ & $2.48 \pm 1.1$ & $2.13 \pm 0.8$\\
    \bottomrule
    \end{tabular}
    \end{sc}
\end{table}

    \begin{table}[ht!]
    \raggedright
    \small
    \caption{CATE estimation performance by embedding model (averages $\pm$ standard deviations).}
    \begin{sc}
    \begin{tabular}{lcccccc}
    \toprule
    \textbf{Embedding Model} & \textbf{$|$Bias$|$ ($\mathbf{X}$)} & \textbf{$|$Bias$|$ ($\mathbf{\widetilde{X}}^{\textrm{DD}}$)} & \textbf{$|$Bias$|$ ($\mathbf{\widetilde{X}}^{\textrm{Drop}}$)} & \textbf{RMSE ($\mathbf{X}$)} & \textbf{RMSE ($\mathbf{\widetilde{X}}^{\textrm{DD}}$)} & \textbf{RMSE ($\mathbf{\widetilde{X}}^{\textrm{Drop}}$)}\\
    \midrule
    EmbeddingGemma-300m & $4.43 \pm 0.7$ & $1.56 \pm 1.0$ & $1.17 \pm 0.8$ & $4.69 \pm 0.5$ & $2.45 \pm 0.9$ & $2.1 \pm 0.8$\\
    all-MiniLM-L6-v2 & $3.99 \pm 0.9$ & $\mathbf{1.20 \pm 0.9}$ & $1.08 \pm 0.7$ & $4.46 \pm 0.7$ & $\mathbf{2.14 \pm 0.9}$ & $2.09 \pm 0.7$\\
    all-mpnet-base-v2 & $\mathbf{3.80 \pm 1.0}$ & $1.47 \pm 1.2$ & $\mathbf{1.07 \pm 0.6}$ & $\mathbf{4.26 \pm 0.7}$ & $2.33 \pm 1.1$ & $\mathbf{2.03 \pm 0.7}$\\
    \bottomrule
    \end{tabular}
    \end{sc}
    \end{table}

\begin{table}[ht!]
    \raggedright
    \small
    \caption{CATE estimation performance by dataset (averages $\pm$ standard deviations).}
    \begin{sc}
    \begin{tabular}{lcccccc}
    \toprule
    \textbf{Dataset} & \textbf{$|$Bias$|$ ($\mathbf{X}$)} & \textbf{$|$Bias$|$ ($\mathbf{\widetilde{X}}^{\textrm{DD}}$)} & \textbf{$|$Bias$|$ ($\mathbf{\widetilde{X}}^{\textrm{Drop}}$)} & \textbf{RMSE ($\mathbf{X}$)} & \textbf{RMSE ($\mathbf{\widetilde{X}}^{\textrm{DD}}$)} & \textbf{RMSE ($\mathbf{\widetilde{X}}^{\textrm{Drop}}$)}\\
    \midrule
    Local Govt Speech & $\mathbf{3.89 \pm 1.0}$ & $\mathbf{1.0 \pm 0.7}$ & $\mathbf{0.78 \pm 0.5}$ & $\mathbf{4.33 \pm 0.8}$ & $\mathbf{1.94 \pm 0.8}$ & $\mathbf{1.74 \pm 0.7}$\\
    Political Ads & $4.38 \pm 0.5$ & $2.14 \pm 1.1$ & $1.37 \pm 0.8$ & $4.67 \pm 0.4$ & $2.95 \pm 1.0$ & $2.27 \pm 0.9$\\
    Us vs. Them & $3.99 \pm 1.0$ & $1.17 \pm 0.8$ & $1.2 \pm 0.6$ & $4.43 \pm 0.7$ & $2.1 \pm 0.9$ & $2.24 \pm 0.6$\\
    \bottomrule
    \end{tabular}
    \end{sc}
    \end{table}

    \begin{table}[ht!]
    \raggedright
    \small
    \caption{CATE estimation performance by layer position (averages $\pm$ standard deviations).}
    \begin{sc}
    \begin{tabular}{lcccccc}
    \toprule
    \textbf{Feature Depth} & \textbf{$|$Bias$|$ ($\mathbf{X}$)} & \textbf{$|$Bias$|$ ($\mathbf{\widetilde{X}}^{\textrm{DD}}$)} & \textbf{$|$Bias$|$ ($\mathbf{\widetilde{X}}^{\textrm{Drop}}$)} & \textbf{RMSE ($\mathbf{X}$)} & \textbf{RMSE ($\mathbf{\widetilde{X}}^{\textrm{DD}}$)} & \textbf{RMSE ($\mathbf{\widetilde{X}}^{\textrm{Drop}}$)}\\
    \midrule
    Early & $\mathbf{3.88 \pm 0.9}$ & $1.59 \pm 1.3$ & $\mathbf{0.95 \pm 0.7}$ & $\mathbf{4.37 \pm 0.6}$ & $2.49 \pm 1.2$ & $\mathbf{1.97 \pm 0.8}$\\
    Middle & $4.27 \pm 0.7$ & $1.42 \pm 0.8$ & $1.15 \pm 0.7$ & $4.6 \pm 0.5$ & $2.34 \pm 0.8$ & $2.08 \pm 0.8$\\
    Late & $4.0 \pm 1.1$ & $\mathbf{1.26 \pm 1.0}$ & $1.17 \pm 0.6$ & $4.39 \pm 0.8$ & $\mathbf{2.12 \pm 1.0}$ & $2.14 \pm 0.7$\\
    \bottomrule
    \end{tabular}
    \end{sc}
    \end{table}

%% file: tables/badic_robustness_tables.tex
\begin{table}[ht!]
    \centering
    \small
    \caption{CATE estimation performance by LLM (averages $\pm$ standard deviations).}
    \begin{sc}
    \begin{tabular}{lcccccc}
    \toprule
    \textbf{LLM} & \textbf{$|$Bias$|$ ($\mathbf{X}$)} & \textbf{$|$Bias$|$ ($\mathbf{\widetilde{X}}^{\textrm{DD}}$)} & \textbf{$|$Bias$|$ ($\mathbf{\widetilde{X}}^{\textrm{Drop}}$)} & \textbf{RMSE ($\mathbf{X}$)} & \textbf{RMSE ($\mathbf{\widetilde{X}}^{\textrm{DD}}$)} & \textbf{RMSE ($\mathbf{\widetilde{X}}^{\textrm{Drop}}$)}\\
    \midrule
    Gemma 2 9B IT & $2.33 \pm 1.9$ & $\mathbf{0.27 \pm 0.2}$ & $\mathbf{0.62 \pm 0.3}$ & $2.91 \pm 1.7$ & $\mathbf{1.18 \pm 0.4}$ & $\mathbf{1.51 \pm 0.5}$\\Llama 3.1 8B Instruct & $\mathbf{1.76 \pm 0.8}$ & $0.47 \pm 0.5$ & $1.07 \pm 0.5$ & $2.88 \pm 0.9$ & $1.57 \pm 0.7$ & $2.24 \pm 0.6$\\Qwen 2.5 7B Instruct & $1.78 \pm 1.2$ & $379.38 \pm 1278.9$ & $0.8 \pm 0.7$ & $\mathbf{2.7 \pm 1.1}$ & $1.74e5 \pm 5.0e5$ & $1.71 \pm 0.8$\\
    \bottomrule
    \end{tabular}
    \end{sc}
    \end{table}

\begin{table}[ht!]
    \centering
    \small
    \caption{CATE estimation performance by embedding model (averages $\pm$ standard deviations).}
    \begin{sc}
    \begin{tabular}{lcccccc}
    \toprule
    \textbf{Embedding Model} & \textbf{$|$Bias$|$ ($\mathbf{X}$)} & \textbf{$|$Bias$|$ ($\mathbf{\widetilde{X}}^{\textrm{DD}}$)} & \textbf{$|$Bias$|$ ($\mathbf{\widetilde{X}}^{\textrm{Drop}}$)} & \textbf{RMSE ($\mathbf{X}$)} & \textbf{RMSE ($\mathbf{\widetilde{X}}^{\textrm{DD}}$)} & \textbf{RMSE ($\mathbf{\widetilde{X}}^{\textrm{Drop}}$)}\\
    \midrule
    EmbeddingGemma-300m  & $2.2 \pm 1.4$ & $219.95 \pm 1058.4$ & $0.99 \pm 0.8$ & $2.98 \pm 1.2$ & $9.07e4 \pm 3.6e5$ & $1.95 \pm 0.8$\\all-MiniLM-L6-v2 & $1.84 \pm 1.3$ & $\mathbf{1.0 \pm 1.4}$ & $0.76 \pm 0.4$ & $2.8 \pm 1.3$ & $\mathbf{2.0 \pm 1.3}$ & $1.77 \pm 0.7$\\all-mpnet-base-v2  & $\mathbf{1.79 \pm 1.3}$ & $206.55 \pm 889.2$ & $\mathbf{0.73 \pm 0.4}$ & $\mathbf{2.68 \pm 1.3}$ & $1.05e5 \pm 4.1e5$ & $\mathbf{1.72 \pm 0.7}$\\
    \bottomrule
    \end{tabular}
    \end{sc}
    \end{table}

    \begin{table}[ht!]
    \centering
    \small
    \caption{CATE estimation performance by dataset (averages $\pm$ standard deviations).}
    \begin{sc}
    \begin{tabular}{lcccccc}
    \toprule
    \textbf{Dataset} & \textbf{$|$Bias$|$ ($\mathbf{X}$)} & \textbf{$|$Bias$|$ ($\mathbf{\widetilde{X}}^{\textrm{DD}}$)} & \textbf{$|$Bias$|$ ($\mathbf{\widetilde{X}}^{\textrm{Drop}}$)} & \textbf{RMSE ($\mathbf{X}$)} & \textbf{RMSE ($\mathbf{\widetilde{X}}^{\textrm{DD}}$)} & \textbf{RMSE ($\mathbf{\widetilde{X}}^{\textrm{Drop}}$)}\\
    \midrule
    Local Govt Speech & $3.14 \pm 1.1$ & $568.13 \pm 1531.9$ & $1.01 \pm 0.7$ & $3.84 \pm 0.8$ & $2.61e5 \pm 6.0e5$ & $2.07 \pm 0.8$\\Political Ads & $\mathbf{0.88 \pm 0.7}$ & $\mathbf{0.38 \pm 0.5}$ & $\mathbf{0.59 \pm 0.4}$ & $\mathbf{1.85 \pm 1.0}$ & $\mathbf{1.32 \pm 0.7}$ & $\mathbf{1.52 \pm 0.7}$\\Us vs. Them & $2.48 \pm 1.1$ & $0.97 \pm 1.3$ & $1.01 \pm 0.5$ & $3.37 \pm 0.9$ & $1.9 \pm 1.2$ & $2.03 \pm 0.6$\\
    \bottomrule
    \end{tabular}
    \end{sc}
    \end{table}

    \begin{table}[ht!]
    \centering
    \small
    \caption{CATE estimation performance by layer position (averages $\pm$ standard deviations).}
    \begin{sc}
    \begin{tabular}{lcccccc}
    \toprule
    \textbf{Feature Depth} & \textbf{$|$Bias$|$ ($\mathbf{X}$)} & \textbf{$|$Bias$|$ ($\mathbf{\widetilde{X}}^{\textrm{DD}}$)} & \textbf{$|$Bias$|$ ($\mathbf{\widetilde{X}}^{\textrm{Drop}}$)} & \textbf{RMSE ($\mathbf{X}$)} & \textbf{RMSE ($\mathbf{\widetilde{X}}^{\textrm{DD}}$)} & \textbf{RMSE ($\mathbf{\widetilde{X}}^{\textrm{Drop}}$)}\\
    \midrule
    Early & $2.38 \pm 1.0$ & $227.84 \pm 1007.8$ & $1.03 \pm 0.6$ & $3.35 \pm 0.8$ & $1.04e5 \pm 4.0e5$ & $2.1 \pm 0.7$\\
    Middle & $4.50 \pm 0.1$ & $0.48 \pm 0.1$ & $0.69 \pm 0.1$ & $4.77 \pm 0.1$ & $1.51 \pm 0.3$ & $1.84 \pm 0.2$\\
    Late & $\mathbf{0.55 \pm 0.4}$ & $\mathbf{0.21 \pm 0.3}$ & $\mathbf{0.46 \pm 0.4}$ & $\mathbf{1.38 \pm 0.5}$ & $\mathbf{1.02 \pm 0.5}$ & $\mathbf{1.24 \pm 0.5}$\\
    \bottomrule
    \end{tabular}
    \end{sc}
    \end{table}

%% file: main.bib
@article{martinvenugopal,
title={Participation and Representation in Local Government Speech},
author={Martin, Olivia and Amar Venugopal},
year={2026}
}

@article{breiman_rf,
title={Random Forests}, url={https://link.springer.com/article/10.1023/A:1010933404324},
journal={Machine Learning},
author={Breiman, Leo},
year={2001},
volume={45},
pages={5-32}
}

@article{chernozhukov_dml,
author = {Chernozhukov, Victor and Chetverikov, Denis and Demirer, Mert and Duflo, Esther and Hansen, Christian and Newey, Whitney and Robins, James},
title = {Double/debiased machine learning for treatment and structural parameters},
journal = {The Econometrics Journal},
volume = {21},
number = {1},
pages = {C1-C68},
year = {2018},
month = {01},
issn = {1368-4221},
doi = {10.1111/ectj.12097},
url = {https://doi.org/10.1111/ectj.12097},
eprint = {https://academic.oup.com/ectj/article-pdf/21/1/C1/27684918/ectj00c1.pdf},
}

@misc{Bills2023Autointerp,
author = {Bills, Steven and Cammarata, Nick and Mossing, Dan and Tillman, Henk and Gao, Leo and Goh, Gabriel and Sutskever, Ilya and Leike, Jan and Wu, Jeff and Saunders, William},
title = {Language models can explain neurons in language models},
url = {https://openai.com/research/language-models-can-explain-neurons-in-language-models},
publisher = {OpenAI},
year = {2023}
}

@misc{cunningham2023sparseautoencodershighlyinterpretable,
      title={Sparse Autoencoders Find Highly Interpretable Features in Language Models}, 
      author={Hoagy Cunningham and Aidan Ewart and Logan Riggs and Robert Huben and Lee Sharkey},
      year={2023},
      eprint={2309.08600},
      archivePrefix={arXiv},
      primaryClass={cs.LG},
      url={https://arxiv.org/abs/2309.08600}, 
}

@article{bricken2023monosemanticity,
title={Towards Monosemanticity: Decomposing Language Models With Dictionary Learning},
author={Bricken, Trenton and Templeton, Adly and Batson, Joshua and Chen, Brian and Jermyn, Adam and Conerly, Tom and Turner, Nick and Anil, Cem and Denison, Carson and Askell, Amanda and Lasenby, Robert and Wu, Yifan and Kravec, Shauna and Schiefer, Nicholas and Maxwell, Tim and Joseph, Nicholas and Hatfield-Dodds, Zac and Tamkin, Alex and Nguyen, Karina and McLean, Brayden and Burke, Josiah E and Hume, Tristan and Carter, Shan and Henighan, Tom and Olah, Christopher},
year={2023},
journal={Transformer Circuits Thread},
url={https://transformer-circuits.pub/2023/monosemantic-features/index.html}
}

@article{templeton2024scaling,
   title={Scaling Monosemanticity: Extracting Interpretable Features from Claude 3 Sonnet},
   author={Templeton, Adly and Conerly, Tom and Marcus, Jonathan and Lindsey, Jack and Bricken, Trenton and Chen, Brian and Pearce, Adam and Citro, Craig and Ameisen, Emmanuel and Jones, Andy and Cunningham, Hoagy and Turner, Nicholas L and McDougall, Callum and MacDiarmid, Monte and Freeman, C. Daniel and Sumers, Theodore R. and Rees, Edward and Batson, Joshua and Jermyn, Adam and Carter, Shan and Olah, Chris and Henighan, Tom},
   year={2024},
   journal={Transformer Circuits Thread},
   url={https://transformer-circuits.pub/2024/scaling-monosemanticity/index.html}
}

@misc{wu2025axbenchsteeringllmssimple,
      title={AxBench: Steering LLMs? Even Simple Baselines Outperform Sparse Autoencoders}, 
      author={Zhengxuan Wu and Aryaman Arora and Atticus Geiger and Zheng Wang and Jing Huang and Dan Jurafsky and Christopher D. Manning and Christopher Potts},
      year={2025},
      eprint={2501.17148},
      archivePrefix={arXiv},
      primaryClass={cs.CL},
      url={https://arxiv.org/abs/2501.17148}, 
}

@misc{kantamneni2025sparseautoencodersusefulcase,
      title={Are Sparse Autoencoders Useful? A Case Study in Sparse Probing}, 
      author={Subhash Kantamneni and Joshua Engels and Senthooran Rajamanoharan and Max Tegmark and Neel Nanda},
      year={2025},
      eprint={2502.16681},
      archivePrefix={arXiv},
      primaryClass={cs.LG},
      url={https://arxiv.org/abs/2502.16681}, 
}

@misc{movva2025sparseautoencodershypothesisgeneration,
      title={Sparse Autoencoders for Hypothesis Generation}, 
      author={Rajiv Movva and Kenny Peng and Nikhil Garg and Jon Kleinberg and Emma Pierson},
      year={2025},
      eprint={2502.04382},
      archivePrefix={arXiv},
      primaryClass={cs.CL},
      url={https://arxiv.org/abs/2502.04382}, 
}

@misc{peng2025usesparseautoencodersdiscover,
      title={Use Sparse Autoencoders to Discover Unknown Concepts, Not to Act on Known Concepts}, 
      author={Kenny Peng and Rajiv Movva and Jon Kleinberg and Emma Pierson and Nikhil Garg},
      year={2025},
      eprint={2506.23845},
      archivePrefix={arXiv},
      primaryClass={cs.LG},
      url={https://arxiv.org/abs/2506.23845}, 
}

@article{feder2022causal,
  title={Causal inference in natural language processing: Estimation, prediction, interpretation and beyond},
  author={Feder, Amir and Keith, Katherine A and Manzoor, Emaad and Pryzant, Reid and Sridhar, Dhanya and Wood-Doughty, Zach and Eisenstein, Jacob and Grimmer, Justin and Reichart, Roi and Roberts, Margaret E and others},
  journal={Transactions of the Association for Computational Linguistics},
  volume={10},
  pages={1138--1158},
  year={2022},
  publisher={MIT Press One Broadway, 12th Floor, Cambridge, Massachusetts 02142, USA~…}
}

@article{egami2022make,
  title={How to make causal inferences using texts},
  author={Egami, Naoki and Fong, Christian J and Grimmer, Justin and Roberts, Margaret E and Stewart, Brandon M},
  journal={Science Advances},
  volume={8},
  number={42},
  pages={eabg2652},
  year={2022},
  publisher={American Association for the Advancement of Science}
}

@misc{vera2025embeddinggemmapowerfullightweighttext,
      title={EmbeddingGemma: Powerful and Lightweight Text Representations}, 
      author={Henrique Schechter Vera and Sahil Dua and Biao Zhang and Daniel Salz and Ryan Mullins and Sindhu Raghuram Panyam and Sara Smoot and Iftekhar Naim and Joe Zou and Feiyang Chen and et. al},
      year={2025},
      eprint={2509.20354},
      archivePrefix={arXiv},
      primaryClass={cs.CL},
      url={https://arxiv.org/abs/2509.20354}, 
}

@misc{wang2020minilmdeepselfattentiondistillation,
      title={MiniLM: Deep Self-Attention Distillation for Task-Agnostic Compression of Pre-Trained Transformers}, 
      author={Wenhui Wang and Furu Wei and Li Dong and Hangbo Bao and Nan Yang and Ming Zhou},
      year={2020},
      eprint={2002.10957},
      archivePrefix={arXiv},
      primaryClass={cs.CL},
      url={https://arxiv.org/abs/2002.10957}, 
}

@misc{song2020mpnetmaskedpermutedpretraining,
      title={MPNet: Masked and Permuted Pre-training for Language Understanding}, 
      author={Kaitao Song and Xu Tan and Tao Qin and Jianfeng Lu and Tie-Yan Liu},
      year={2020},
      eprint={2004.09297},
      archivePrefix={arXiv},
      primaryClass={cs.CL},
      url={https://arxiv.org/abs/2004.09297}, 
}

@inproceedings{veitch2020adapting,
  title={Adapting text embeddings for causal inference},
  author={Veitch, Victor and Sridhar, Dhanya and Blei, David},
  booktitle={Conference on uncertainty in artificial intelligence},
  pages={919--928},
  year={2020},
  organization={PMLR}
}

@misc{imai2025causalrepresentationlearninggenerative,
      title={Causal Representation Learning with Generative Artificial Intelligence: Application to Texts as Treatments}, 
      author={Kosuke Imai and Kentaro Nakamura},
      year={2025},
      eprint={2410.00903},
      archivePrefix={arXiv},
      primaryClass={stat.AP},
      url={https://arxiv.org/abs/2410.00903}, 
}

@article{modarressi2025causal,
  title={Causal inference on outcomes learned from text},
  author={Modarressi, Iman and Spiess, Jann and Venugopal, Amar},
  journal={arXiv preprint arXiv:2503.00725},
  year={2025}
}

@article{huang2024ravel,
  title={RAVEL: Evaluating interpretability methods on disentangling language model representations},
  author={Huang, Jing and Wu, Zhengxuan and Potts, Christopher and Geva, Mor and Geiger, Atticus},
  journal={arXiv preprint arXiv:2402.17700},
  year={2024}
}

@article{gui2022causal,
  title={Causal estimation for text data with (apparent) overlap violations},
  author={Gui, Lin and Veitch, Victor},
  journal={arXiv preprint arXiv:2210.00079},
  year={2022}
}

@inproceedings{pryzant2021causal,
  title={Causal effects of linguistic properties},
  author={Pryzant, Reid and Card, Dallas and Jurafsky, Dan and Veitch, Victor and Sridhar, Dhanya},
  booktitle={Proceedings of the 2021 Conference of the North American Chapter of the Association for Computational Linguistics: Human Language Technologies},
  pages={4095--4109},
  year={2021}
}

@article{mozer2020matching,
  title={Matching with text data: An experimental evaluation of methods for matching documents and of measuring match quality},
  author={Mozer, Reagan and Miratrix, Luke and Kaufman, Aaron Russell and Anastasopoulos, L Jason},
  journal={Political Analysis},
  volume={28},
  number={4},
  pages={445--468},
  year={2020},
  publisher={Cambridge University Press}
}

@article{lieberum2024gemma,
  title={Gemma scope: Open sparse autoencoders everywhere all at once on gemma 2},
  author={Lieberum, Tom and Rajamanoharan, Senthooran and Conmy, Arthur and Smith, Lewis and Sonnerat, Nicolas and Varma, Vikrant and Kram{\'a}r, J{\'a}nos and Dragan, Anca and Shah, Rohin and Nanda, Neel},
  journal={arXiv preprint arXiv:2408.05147},
  year={2024}
}

@inproceedings{huguet-cabot-etal-2021-us,
    title = "Us vs. Them: A Dataset of Populist Attitudes, News Bias and Emotions",
    author = "Huguet Cabot, Pere-Llu{\'i}s  and
      Abadi, David  and
      Fischer, Agneta  and
      Shutova, Ekaterina",
    editor = "Merlo, Paola  and
      Tiedemann, Jorg  and
      Tsarfaty, Reut",
    booktitle = "Proceedings of the 16th Conference of the European Chapter of the Association for Computational Linguistics: Main Volume",
    month = apr,
    year = "2021",
    address = "Online",
    publisher = "Association for Computational Linguistics",
    url = "https://aclanthology.org/2021.eacl-main.165/",
    doi = "10.18653/v1/2021.eacl-main.165",
    pages = "1921--1945"
}

@misc{sobhani2024multienvironmenttopicmodels,
      title={Multi-environment Topic Models}, 
      author={Dominic Sobhani and Amir Feder and David Blei},
      year={2024},
      eprint={2410.24126},
      archivePrefix={arXiv},
      primaryClass={cs.CL},
      url={https://arxiv.org/abs/2410.24126}, 
}

@misc{neuronpedia,
    title = {Neuronpedia: Interactive Reference and Tooling for Analyzing Neural Networks},
    year = {2023},
    note = {Software available from neuronpedia.org},
    url = {https://www.neuronpedia.org},
    author = {Lin, Johnny}
}

@misc{gemmateam2024gemma2improvingopen,
      title={Gemma 2: Improving Open Language Models at a Practical Size}, 
      author={Gemma Team and Morgane Riviere and Shreya Pathak and Pier Giuseppe Sessa and Cassidy Hardin and Surya Bhupatiraju and Léonard Hussenot and Thomas Mesnard and Bobak Shahriari and Alexandre Ramé and and et. al},
      year={2024},
      eprint={2408.00118},
      archivePrefix={arXiv},
      primaryClass={cs.CL},
      url={https://arxiv.org/abs/2408.00118}, 
}

@misc{qwen2025qwen25technicalreport,
      title={Qwen2.5 Technical Report}, 
      author={Qwen and : and An Yang and Baosong Yang and Beichen Zhang and Binyuan Hui and Bo Zheng and Bowen Yu and Chengyuan Li and Dayiheng Liu and Fei Huang and et. al},
      year={2025},
      eprint={2412.15115},
      archivePrefix={arXiv},
      primaryClass={cs.CL},
      url={https://arxiv.org/abs/2412.15115}, 
}

@misc{grattafiori2024llama3herdmodels,
      title={The Llama 3 Herd of Models}, 
      author={Aaron Grattafiori and Abhimanyu Dubey and Abhinav Jauhri and Abhinav Pandey and Abhishek Kadian and Ahmad Al-Dahle and Aiesha Letman and Akhil Mathur and Alan Schelten and Alex Vaughan and et. al},
      year={2024},
      eprint={2407.21783},
      archivePrefix={arXiv},
      primaryClass={cs.AI},
      url={https://arxiv.org/abs/2407.21783}, 
}

@misc{comanici2025gemini25pushingfrontier,
      title={Gemini 2.5: Pushing the Frontier with Advanced Reasoning, Multimodality, Long Context, and Next Generation Agentic Capabilities}, 
      author={Gheorghe Comanici and Eric Bieber and Mike Schaekermann and Ice Pasupat and Noveen Sachdeva and Inderjit Dhillon and Marcel Blistein and Ori Ram and Dan Zhang and Evan Rosen and et. al},
      year={2025},
      eprint={2507.06261},
      archivePrefix={arXiv},
      primaryClass={cs.CL},
      url={https://arxiv.org/abs/2507.06261}, 
}

@misc{arditi2024finding,
  title = {Finding Misaligned Persona Features in Open-Weight Models},
  author = {Arditi, Andy},
  howpublished = {LessWrong},
  url = {https://www.lesswrong.com/posts/NCWiR8K8jpFqtywFG/finding-misaligned-persona-features-in-open-weight-models},
  year = {2024},
  month = {October}
}

@inproceedings{fong2016discovery,
  title={Discovery of treatments from text corpora},
  author={Fong, Christian and Grimmer, Justin},
  booktitle={Proceedings of the 54th Annual Meeting of the Association for Computational Linguistics (Volume 1: Long Papers)},
  pages={1600--1609},
  year={2016}
}

@article{ash2023text,
  title={Text algorithms in economics},
  author={Ash, Elliott and Hansen, Stephen},
  journal={Annual Review of Economics},
  volume={15},
  pages={659--688},
  year={2023},
  publisher={Annual Reviews}
}

@article{angelopoulos2024value,
  title={Value aligned large language models},
  author={Angelopoulos, Panagiotis and Lee, Kevin and Misra, Sanjog},
  journal={Available at SSRN},
  year={2024}
}

@article{zheng2025model,
  title={Model Directions, Not Words: Mechanistic Topic Models Using Sparse Autoencoders},
  author={Zheng, Carolina and Beltran-Velez, Nicolas and Karlekar, Sweta and Shi, Claudia and Nazaret, Achille and Mallik, Asif and Feder, Amir and Blei, David M},
  journal={arXiv preprint arXiv:2507.23220},
  year={2025}
}

@misc{econml,
  author={Battocchi, Keith and Dillon, Eleanor and Hei, Maggie and Lewis, Greg and Oka, Paul and Oprescu, Miruna and Syrgkanis, Vasilis},
  title={{EconML}: {A Python Package for ML-Based Heterogeneous Treatment Effects Estimation}},
  howpublished={\url{https://github.com/py-why/EconML}},
  year={2019}
}

@article{nie2021quasi,
  title={Quasi-oracle estimation of heterogeneous treatment effects},
  author={Nie, Xinkun and Wager, Stefan},
  journal={Biometrika},
  volume={108},
  number={2},
  pages={299--319},
  year={2021},
  publisher={Oxford University Press}
}

@article{gerber2011large,
  title={How large and long-lasting are the persuasive effects of televised campaign ads? Results from a randomized field experiment},
  author={Gerber, Alan S and Gimpel, James G and Green, Donald P and Shaw, Daron R},
  journal={American Political Science Review},
  volume={105},
  number={1},
  pages={135--150},
  year={2011},
  publisher={Cambridge University Press}
}

@article{fenerty2012effect,
  title={The effect of reminder systems on patients’ adherence to treatment},
  author={Fenerty, Sarah D and West, Cameron and Davis, Scott A and Kaplan, Sebastian G and Feldman, Steven R},
  journal={Patient preference and adherence},
  pages={127--135},
  year={2012},
  publisher={Taylor \& Francis}
}

@book{egan2013skilled,
  title={The skilled helper: A problem-management and opportunity-development approach to helping},
  author={Egan, Gerard},
  year={2013},
  publisher={Nelson Education}
}

@book{green2019get,
  title={Get out the vote: How to increase voter turnout},
  author={Green, Donald P and Gerber, Alan S},
  year={2019},
  publisher={Brookings Institution Press}
}

@article{luong2021text,
  title={Text message medication adherence reminders automated and delivered at scale across two institutions: testing the nudge system: pilot study},
  author={Luong, Phat and Glorioso, Thomas J and Grunwald, Gary K and Peterson, Pamela and Allen, Larry A and Khanna, Amber and Waughtal, Joy and Sandy, Lisa and Ho, P Michael and Bull, Sheana},
  journal={Circulation: Cardiovascular Quality and Outcomes},
  volume={14},
  number={5},
  pages={e007015},
  year={2021},
  publisher={Am Heart Assoc}
}

@article{ellickson2023estimating,
  title={Estimating marketing component effects: Double machine learning from targeted digital promotions},
  author={Ellickson, Paul B and Kar, Wreetabrata and Reeder III, James C},
  journal={Marketing Science},
  volume={42},
  number={4},
  pages={704--728},
  year={2023},
  publisher={INFORMS}
}
